\pdfoutput=1

\documentclass[11pt]{article}

\usepackage[]{acl}

\usepackage{times}
\usepackage{latexsym}
\usepackage{multirow}
\usepackage{booktabs}
\usepackage{graphicx}
\usepackage{amssymb}
\usepackage{makecell}
\usepackage{amsmath, bm}
\usepackage{longtable}
\usepackage{subfigure}
\usepackage{caption}
\usepackage{enumerate}
\usepackage{enumitem}
\usepackage{longtable}
\usepackage{xspace}
\usepackage{comment}
\usepackage{cuted}
\usepackage{url}
\usepackage{amssymb}
\usepackage{pifont}

\usepackage{times}
\usepackage{helvet}
\usepackage{courier}
\usepackage{xcolor}
\usepackage{tcolorbox}
\usepackage{amsmath}

\usepackage{amsmath}
\usepackage{amssymb}
\usepackage{mathtools}
\usepackage{amsthm}
\usepackage{multirow}
\usepackage[capitalize,noabbrev]{cleveref}
\usepackage{fontawesome}

\newtheorem{theorem}{Theorem}[section]

\newtheorem{definition}[theorem]{Definition}

\usepackage[T1]{fontenc}

\usepackage[utf8]{inputenc}

\usepackage{microtype}

\usepackage{inconsolata}

\title{Understanding Generalization in Role-Playing Models\\ via Information Theory}

\author{
\textbf{Yongqi Li$^{1,2,*,\dagger}$}, 
\textbf{Hao Lang$^{2,\dagger}$}, 
\textbf{Fei Huang$^{2}$}, 
\textbf{Tieyun Qian$^{1,3,\ddag}$},
\textbf{Yongbin Li$^{2,\ddag}$}\\
        $^1$ School of Computer Science, Wuhan University, $^2$ Tongyi Lab\\
        $^3$ Zhongguancun Academy\\
        \small{\texttt{\{liyongqi,qty\}@whu.edu.cn}, \texttt{\{hao.lang,f.huang,shuide.lyb\}@alibaba-inc.com}}
        }

\begin{document}
\maketitle

\renewcommand{\thefootnote}{\fnsymbol{footnote}}
\footnotetext[1]{{ }Work done while the author was interning at Tongyi Lab.}
\footnotetext[2]{{ }Equal contributions.}
\footnotetext[3]{{ }Corresponding authors.}
\renewcommand{\thefootnote}{\arabic{footnote}}


\begin{abstract}
Role-playing models (RPMs) are widely used in real-world applications but underperform when deployed in the wild.
This degradation can be attributed to distribution shifts, including user, character, and dialogue compositional shifts.
Existing methods like LLM-as-a-judge fall short in providing a fine-grained diagnosis of how these shifts affect RPM generalization, and thus there lack formal frameworks to characterize RPM generalization behaviors.
To bridge these gaps, we introduce an information-theoretic metric, named reasoning-based effective mutual information difference (R-EMID), to measure RPM performance degradation in an interpretable way.
We also derive an upper bound on R-EMID to predict the worst-case generalization performance of RPMs and theoretically reveal how various shifts contribute to the RPM performance degradation.
Moreover, we propose a co-evolving reinforcement learning framework to adaptively model the connection among user, character, and dialogue context and thus enhance the estimation of dialogue response generation probability, which is critical for calculating R-EMID.
Finally, we evaluate the generalization performance of various RPMs using R-EMID, finding that user shift poses the highest risk among all shifts and reinforcement learning is the most effective approach for enhancing RPM generalization. Code and data are available at \url{https://github.com/AlibabaResearch/DAMO-ConvAI/tree/main/RPM-Generalization}.
\end{abstract}



\section{Introduction}
As large language models (LLMs)~\cite{openai-2024-gpt4,anthropic-2022-claude} become more capable, the development of role-playing models (RPMs)~\cite{chen-2024-oscarsSurvey,chen-2024-personaSurvey} has emerged as a key direction in advancing human-AI interaction. 
By simulating characters with distinct personalities, RPMs enable emotionally rich and contextually grounded dialogues, making them particularly valuable in fields such as entertainment~\cite{park-2023-generative}, personalized education~\cite{pataranutaporn-2021-ai}, and emotional companionship~\cite{liu-2021-towardsemotional}. 

In real-world applications, the generalization abilities of RPMs are crucial for user experiences.
For instance, global-facing platforms such as \citet{characterAI-2022-characterAI} serve users from diverse unseen linguistic and cultural backgrounds, and users may also request RPMs to simulate previously unseen characters. 
However, recent studies have reported failures in RPMs, such as generating culturally inappropriate responses~\cite{li-2024-culturepark} or exhibiting character inconsistency~\cite{wang-2025-opencharacter}, raising concerns about their reliability.

These failure cases suggest that current RPMs underperform when deployed in the wild, which can be attributed to \textit{distribution shifts} between training and test data.
We categorize distribution shifts in RPMs into three types: user shift (e.g., unseen linguistic or cultural backgrounds), character shift (e.g., unseen character types), and dialogue compositional shift (e.g., longer and more complex multi-turn interactions).
Understanding how these shifts contribute to RPM performance degradation is essential for building systems that generalize reliably in real‑world applications. 

Although existing methods, such as LLM-as-a-judge~\cite{tu-2024-charactereval}, can empirically quantify RPM performance degradation under distribution shifts, they fall short in providing a fine-grained diagnosis of how different shifts contribute to the degradation. 
More importantly, without a formal framework connecting performance degradation to distribution shifts, it is difficult to gain a systematic understanding of RPM generalization. 
This gap motivates our research question: \textit{\textbf{Can we develop a theoretical framework that formally characterizes how user, character, and dialogue compositional shift impact RPM generalization?}}

To address this, we propose an information-theoretic metric to quantify RPM performance degradation under distribution shifts, named reasoning-based effective mutual information difference (R-EMID).
Concretely, we build R-EMID based on EMID~\cite{oh-2025-understanding}, which is the first interpretable information-theoretic method for measuring the generalization performance of foundation models. 
To accurately estimate the conditional response generation probability required by EMID, we introduce an additional reasoning process to infer connections between user persona, agent character, dialogue history, and the target response. 
Moreover, we derive an upper bound for R-EMID, which enables the prediction of worst-case generalization performance of RPMs and theoretically reveals how various shifts contribute to the RPM performance drop.

To robustly generate the chain-of-thought for the reasoning phase required by R-EMID, we propose a co-evolving reinforcement learning (CoRL) framework to iteratively optimize a reasoning generator and a dialogue policy model. 
Specifically, the reasoning generator generates thoughts that help the current policy model select useful information from user, character, and dialogue history for conditional probability estimation, while the policy model produces dialogue responses and provides feedback for optimizing the reasoning generator and keeping it in-distribution.
Extensive experimental results show strong correlations between R-EMI/R-EMID and LLM-as-a-judge metrics, as well as correlations between R-EMID and its upper bounds, which validate the proposed R-EMI/R-EMID metrics and CoRL.

Finally, we evaluate the generalization performance of various RPMs using R-EMID.
Our analysis yields three findings:
1) Among the shifts considered, the user shift poses the highest risk to RPM reliability.
2) Reinforcement learning offers consistent generalization gains over standard supervised fine‑tuning, whereas data augmentation delivers inconsistent results.
3) Naively incorporating reasoning traces fails to improve RPM generalization.

In this work, we make the following key contributions:
1) We identify three types of distribution shifts in RPMs, including user, character, and dialogue compositional shifts.
2) We introduce the first information-theoretic framework, named R-EMID, for characterizing RPM performance drop under various shifts. Additionally, we propose CoRL for estimating R-EMID and validate its effectiveness on real-world datasets. 
3) We systematically evaluate the generalization performance of various RPMs using R-EMID, providing several valuable insights.


\section{Background}
\subsection{Problem Formulation}
Let $\mathcal{X}$ and $\mathcal{Y}$ denote the input and output spaces of role-playing models (RPMs), respectively.
We define variables $X = (X_u, X_a, X_d) \in \mathcal{X}$ and $Y \in \mathcal{Y}$, where $X$ represents the input consisting of user persona $X_u$, agent character $X_a$, and dialogue context $X_d$, while $Y$ denotes the response. 
We denote the joint distribution as $P_{XY}$, with marginals $P_{X}$, $P_{Y}$, and conditional distribution $P_{Y|X}$.
In subsequent sections, $P_{XY}$ refers to the training data distribution (in-distribution, ID) while $Q_{XY}$ refers to the test distribution (out-of-distribution, OOD).

Let $P_{\theta}: \mathcal{X} \to \mathcal{Y}$ be an RPM parameterized by $\theta$, trained on ID data. Given an input $X \in \mathcal{X}$, the model produces a response $Y_{\theta} \in \mathcal{Y}$. The marginal distribution of the model responses on ID and OOD data is given by $P_{Y_{\theta}} = \mathbb{E}_{x \sim P_X} \left[ P_{\theta}(\cdot|x) \right]$ and $Q_{Y_{\theta}} = \mathbb{E}_{x \sim Q_X} \left[ P_{\theta}(\cdot|x) \right]$, repectively.

\begin{figure}[t]
\centering
\includegraphics[width=0.42\textwidth]{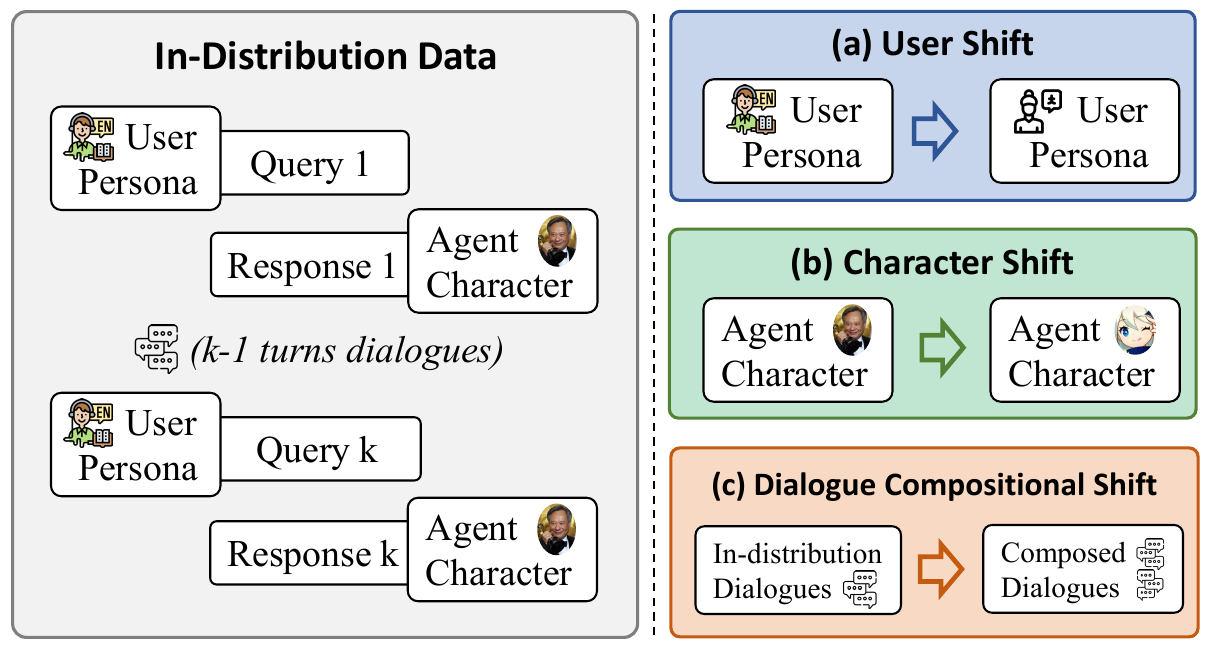}
\caption{Distribution shifts in RPMs. (a) User shift: distribution change in user persona, e.g., \texttt{EN$\rightarrow$ZH}; (b) Character shift: distribution change in agent character, e.g., unseen character types; (c) Dialogue compositional shift: distribution change in dialogues, e.g., from short-turn dialogue to composed long-turn dialogue.}
\label{fig:shift_definition}
\end{figure}


\begin{figure*}[t]
\centering
\includegraphics[width=0.9\textwidth]{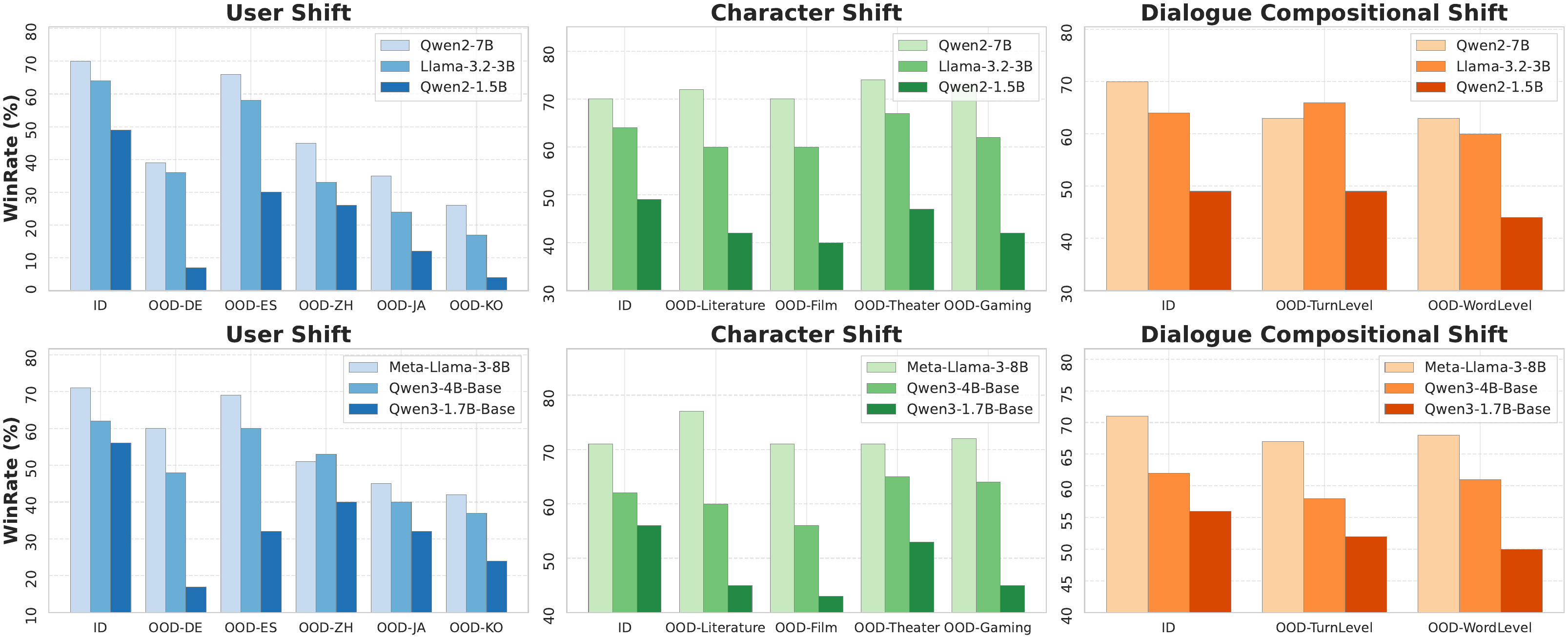}
\caption{Pilot study of RPM performance under various distribution shifts, including user shift (left), character shift (middle), and dialogue compositional shift (right). The RPMs are trained with supervised fine-tuning based on six LLMs (extended results based on more LLMs are provided in Appendix~\ref{sec:app:pilot_study_more_results}).
Here, the x-axis is sorted by the severity of the distribution shifts, while the y-axis represents RPM performance measured by \textit{WinRate}.
We observe a consistent trend: as the severity of the distribution shift increases, performance degrades more significantly.}
\label{fig:djs_drop_compare}
\end{figure*}

\subsection{Distribution Shifts in RPMs}\label{sec:shift_definition}
Despite the expectation that RPMs should robustly handle unseen users, novel characters, and evolving dialogues in real-world deployment, recent reports of failures in these edge cases have raised concerns about their reliability~\cite{roose-2024-can}.
We argue that these failures can be attributed to the poor generalization when the test data deviates from the training data, i.e., under \textit{distribution shifts}.

To understand these failure modes, as shown in Figure~\ref{fig:shift_definition}, we first identify three types of distribution shifts that RPMs may encounter in real-world deployments:

\begin{itemize}
    \item \textbf{\textit{User Shift}}: the marginal distribution of user persona undergoes shift $D(P_{X_u}||Q_{X_u})\gg 0$.
    \item \textbf{\textit{Character Shift}}: the marginal distribution of agent character exhibits substantial divergence $D(P_{X_a}||Q_{X_a})\gg 0$.
    \item \textbf{\textit{Dialogue Compositional Shift}}: the marginal distribution of dialogue exhibits shift $D(P_{X_d}||Q_{X_d})\gg 0$. By ``compositional'', we mean that test dialogues are constructed by recombining turns or words from ID dialogues.
\end{itemize}

Here, $D(\cdot||\cdot)$ denotes a divergence metric that quantifies the discrepancy between the training distribution $P$ and the test distribution $Q$. In this paper, we adopt the Jensen-Shannon (JS) divergence~\cite{menendez-1997-JSD} as our choice of $D(\cdot||\cdot)$.
Based on these definitions, this work aims to \textit{understand RPM behaviors, especially performance drop, under these distribution shifts}.

\subsection{Pilot Study}\label{sec:pilot_study}
A natural hypothesis is that as the degree of distribution discrepancy increases, the magnitude of RPM performance degradation under distribution shifts also increases. To test this hypothesis, we first conduct an empirical pilot study to evaluate the performance drop of RPMs under the three types of distribution shifts described above.

Since no existing datasets can systematically evaluate RPM performance under the above distribution shifts, we introduce \textbf{RPGBench}, a novel benchmark constructed according to the shift definitions in Sec.~\ref{sec:shift_definition}. The RPGBench comprises a total of 17k samples, including 5k ID samples, 5k OOD samples for user shift, 6k OOD samples for character shift, and 1k OOD samples for dialogue compositional shift. Each sample contains three input components: agent character, user persona, and dialogue context, and a reference response for training or evaluation~\footnote{The dataset is available at \url{https://huggingface.co/datasets/YongqiLi/RPGBench}.}.

Specifically, first, for user shift, we collect user personas described in English to construct the ID data. The OOD test data are constructed based on user personas with five non-English linguistic and cultural backgrounds, including \textit{German, Spanish, Chinese, Japanese, and Korean}. Second, for character shift, we construct the ID data using real-world characters, while the OOD data are built from fictional characters from \textit{literature, film, theater, and gaming} domains. Third, for dialogue compositional shift, OOD dialogues are crafted by recombining turns or words from ID dialogues. At the turn level, dialogues shift from 4 turns to 8 turns, while at the word level, test queries are rewritten using words in the training queries.
Please refer to Appendix~\ref{sec:app:bench_details} for construction details.

We select eleven LLMs from the Qwen~\cite{yang-2025-qwen3} and LLaMA~\cite{meta-2023-llama3} series, and fine-tune them on the ID training data of RPGBench via supervised fine-tuning.
The RPM performance is measured by \textit{WinRate}, which is an LLM-as-a-judge metric that is commonly adopted in previous evaluation methods for role-playing tasks~\cite{wang-2024-rolellm}. 
We calculate the \textit{WinRate} by comparing the model response with the reference response in the dataset.
To reduce potential bias from individual LLM judges~\cite{zheng-2023-judging}, we adopt an ensemble strategy.
Specifically, we average scores from multiple state-of-the-art LLMs, including Qwen3-235B-A22B, GPT-4o-mini, and Claude-3.7-Sonnet, as the final score.

Figure~\ref{fig:djs_drop_compare} shows the performance drop (y-axis) of RPMs built on different LLMs, evaluated across 11 shift subtypes, covering user, character, and dialogue shifts.
The x-axis is sorted by the severity of shifts, measured by $D_{JS}(P_{X}||Q_{X})$.
From the figure, we observe a trend: as the severity of the distribution shift increases, the magnitude of performance degradation increases.
\textit{These empirical observations suggest the potential existence of a fundamental principle connecting performance drop with distributional shifts in RPMs, which motivates the following theoretical investigation into the underlying mechanisms.
}

\section{Understanding RPM Generalization via Information Theory}
In this section, we will introduce: 1) the existing EMID (Sec.~\ref{sec:theory_pre}); 2) the proposed R-EMID to quantify the performance degradation of role-playing models (RPMs) (Sec .~\ref {sec:theory_our_REMID}); 3) the proposed CoRL for estimating R-EMID (Sec .~\ref {sec:theory_REMID_CoEvolve}); and 4) experiments to validate R-EMID and CoRL (Sec .~\ref {sec:empirical_validation}).

\subsection{Preliminary}\label{sec:theory_pre}
This section introduces effective mutual information difference (EMID)~\cite{oh-2025-understanding}, the first information-theoretic metric for measuring the generalization performance of foundation models.
\begin{definition}[{EMID}]\label{definition:EMID_definition}
The EMID of a model $P_{\theta}$ with parameters $\theta$ is defined as:
\begin{equation*}
\footnotesize
\begin{split}
\text{EMID}(P_{XY}, Q_{XY}; P_{\theta}):=\text{EMI}(P_{XY}; P_{\theta}) - \text{EMI}(Q_{XY}; P_{\theta})\\
:= \left( I(P_{XY_{\theta}}) - I(P_{XY}) \right) - \left( I(Q_{XY_{\theta}}) - I(Q_{XY}) \right),
\end{split}
\end{equation*}
where EMI measures the quality of model responses by quantifying how much they deviate from the ground-truth ones.
The standard mutual information $I(P_{XY})$ reflects the degree of dependence between $P_{X}$ and $P_{Y}$, as defined by~\citet{shannon-1948-mathematical}.
\end{definition}

EMID can characterize the model performance drop under distribution shifts from $P$ (training data) to $Q$ (test data). Compared to black-box evaluation methods, such as LLM-as-a-judge~\cite{zheng-2023-judging}, EMID offers two key advantages: 1) It is theoretically grounded and enjoys strong interpretability. 2) The theoretical upper bound of EMID is derivable and can characterize the model's maximum generalization risk, which is important in safety-critical applications~\cite{yu-2024-OODsurvey}.

\subsection{Reasoning-based EMID}\label{sec:theory_our_REMID}

While EMID offers a principled way to quantify model performance degradation under distribution shifts, its direct application to RPMs faces significant challenges. Specifically, estimating EMID relies on mutual information estimators such as CLUB~\cite{cheng-2020-club}, which requires accurately estimating the response generation probability conditioned on the input, denoted as $p(y|x)$. 
However, the input of RPMs is inherently heterogeneous, comprising user persona, agent character, and dialogue context.
Directly modeling such heterogeneous inputs is quite challenging and may lead to information loss, making the accurate estimation of $p(y|x)$ difficult.

To address this, we propose \emph{reasoning-based EMID} (R-EMID), which introduces an intermediate reasoning process $R$ to infer the connection between user, character, dialogue, and the targeted response. 
For instance, $R$ may encode information about what stylistic and content features an ideal response should exhibit.
This reasoning process enables the extraction of heterogeneous information from $X$, without the need to explicitly model the complex dependencies in $X$.
By introducing the reasoning process $R$, R-EMID offers a more tractable way for measuring RPM performance degradation under distribution shifts.

\begin{definition}[{Reasoning-based EMI, R-EMI}]\label{definition:REMI}
Given a RPM $P_{\theta}$ parameterized by $\theta$, its R-EMI under a joint distribution $P_{XY}$ is defined as:
\begin{equation*}
\footnotesize
\begin{split}
    \text{R-EMI}(P_{XY}; P_{\theta}) := I(P_{X_{R}Y_{\theta}}) - I(P_{X_{R}Y}),
\end{split}
\end{equation*}
where the definition of mutual information $I(\cdot)$ is the same as in Definition~\ref{definition:EMID_definition}.
The newly introduced variable $X_R$ is obtained via a mapping function $f_R(\cdot)$, which augments the input $X$ into the $X_R=(X,R=f_R(X))$.
Similar to EMI, R-EMI also aims to quantify the quality of responses produced by the model $P_{\theta}$ under certain distributions.
\end{definition}

Now, based on R-EMI, we can quantify the performance drop of RPMs with formal guarantees via reasoning-based EMID (R-EMID).
\begin{definition}[{Reasoning-based EMID, R-EMID}]\label{definition:REMID}
We define R-EMID as the difference between the R-EMI on the ID distribution and OOD distribution: $\text{R-EMID}(P_{XY}, Q_{XY}; P_{\theta}):=\text{R-EMI}(P_{XY}; P_{\theta}) - \text{R-EMI}(Q_{XY}; P_{\theta})$. R-EMID can be used to characterize the performance degradation of RPMs under distribution shifts from $P_{XY}$ to $Q_{XY}$.
\end{definition}

Furthermore, we derive an upper bound on R-EMID in Theorem~\ref{theorem:REMID_upperbound}, which characterizes the maximum performance drop of RPMs under distribution shifts, enabling the maximum risk analysis in safety-critical applications.
Moreover, this bound principally reveals the underlying connection between the RPM performance drop and the severity of user, character, and dialogue shifts.

\begin{theorem}[Upper Bound on R-EMID]\label{theorem:REMID_upperbound}
Given an RPM $P_{\theta}$ trained on the distribution $P_{XY}$ and tested on $Q_{XY}$, which share consistent conditional distributions $Y|X$. If  there exist constants $\delta_{P}$ and $\delta_{Q}$ such that $D_{\rm JS}(P_{Y_{{\theta}}}\|P_{Y})\leq \delta_{P}, D_{\rm JS}(Q_{Y_{{\theta}}}\|Q_{Y})\leq \delta_{Q}, \Delta = \delta_{P} + \delta_{Q}$, then the $\text{R-EMID}(P_{XY}, Q_{XY}; P_{\theta})$ is upper bounded by:
\begin{equation}\label{eq:REMID_bound_simple}
\footnotesize
\sqrt{2/3}~\widehat{H} \sum_{z \in \{u, a, d\}} D^{1/2}_{\rm JS}(P_{X_z} \| Q_{X_z}) + 8\Delta^{1/4},
\end{equation}
where $\widehat{H}=\max_{x \in \mathcal{X}} [H(Q_{Y|x,r})+H(P_{Y_{\theta}|x,r})]$, and $r=f_R(x)$ denotes the reasoning process~\footnote{Please refer to Appendix~\ref{sec:app:theory_details} for an extended theoretical analysis with full proof.}.
\end{theorem}

\begin{tcolorbox}[
    width=\linewidth,
    colback=white!96!black,  
    boxrule=1pt,           
    colframe=black!60,       
    arc=2pt                 
]
\noindent \textbf{Implication.}
Theorem~\ref{theorem:REMID_upperbound} implies that the RPM performance drop is affected by the combined contribution of marginal distribution shifts across user ($u$), agent character ($a$), and dialogue ($d$) components, as quantified by $\sum_{z \in \{u, a, d\}} D^{1/2}_{\mathrm{JS}}(P_{X_z} \| Q_{X_z})$. Additionally, it is influenced by the maximum model response uncertainty, i.e., $\hat{H}$, and the maximum deviation between model predictions and the reference responses, i.e., $\Delta$.
\end{tcolorbox}


\subsection{R-EMID Estimation}\label{sec:theory_REMID_CoEvolve}
Based on Definitions~\ref{definition:REMI} and~\ref{definition:REMID}, the estimation of R-EMI and R-EMID relies on estimating mutual information terms.
For example, to estimate mutual information $I(P_{X_R Y})$, we first need to collect $N$ samples $\{(x_i, r_i, y_i)\}_{i=1}^{N}$, where each $(x_i, y_i)$ can be directly sampled from $P_{XY}$, while the reasoning process $r_i$ needs to be obtained via a mapping function $f_R$, i.e., $r_i=f_R(x_i)$.
Then, the estimated $\hat{I}(P_{X_R Y})$ can then be calculated as:
\begin{equation}\label{eq:RMI_club_estimation}
\footnotesize
    \frac{1}{N} \sum_{i=1}^{N} \log p(y_i | x_i, r_i) - \frac{1}{N^2} \sum_{i=1}^{N} \sum_{j=1}^{N} \log p(y_j | x_i, r_i).
\end{equation}

In practice, both the reasoning mapping function $f_R$ and the conditional distribution $p(y|x,r)$ are inaccessible. Therefore, as shown in Figure~\ref{fig:CoEvolveRL}, we approximate them using two LLM-based parametric functions: a reasoning generator $q_{\phi_1}(\cdot)$ and a policy model $q_{\phi_2}(\cdot|\cdot)$.

To optimize the reasoning generator $q_{\phi_1}(\cdot)$ and policy model $q_{\phi_2}(\cdot|\cdot)$, we propose a \textbf{co-evolving reinforcement learning (CoRL)} framework (Figure~\ref{fig:CoEvolveRL}).
During co-evolution, the policy model $q_{\phi_2}$ provides rewards for optimizing the reasoning generator $q_{\phi_1}$ to keep it in distribution. The reasoning generator generates reasoning processes $r$ to assist the current policy model in selecting useful information from the input $x$ for estimating the conditional probability.
Note that CoRL is only used during the training stage, and the trained model can be reused indefinitely.

\vspace{0mm}
\begin{figure}[t]
\centering
\includegraphics[width=0.47\textwidth]{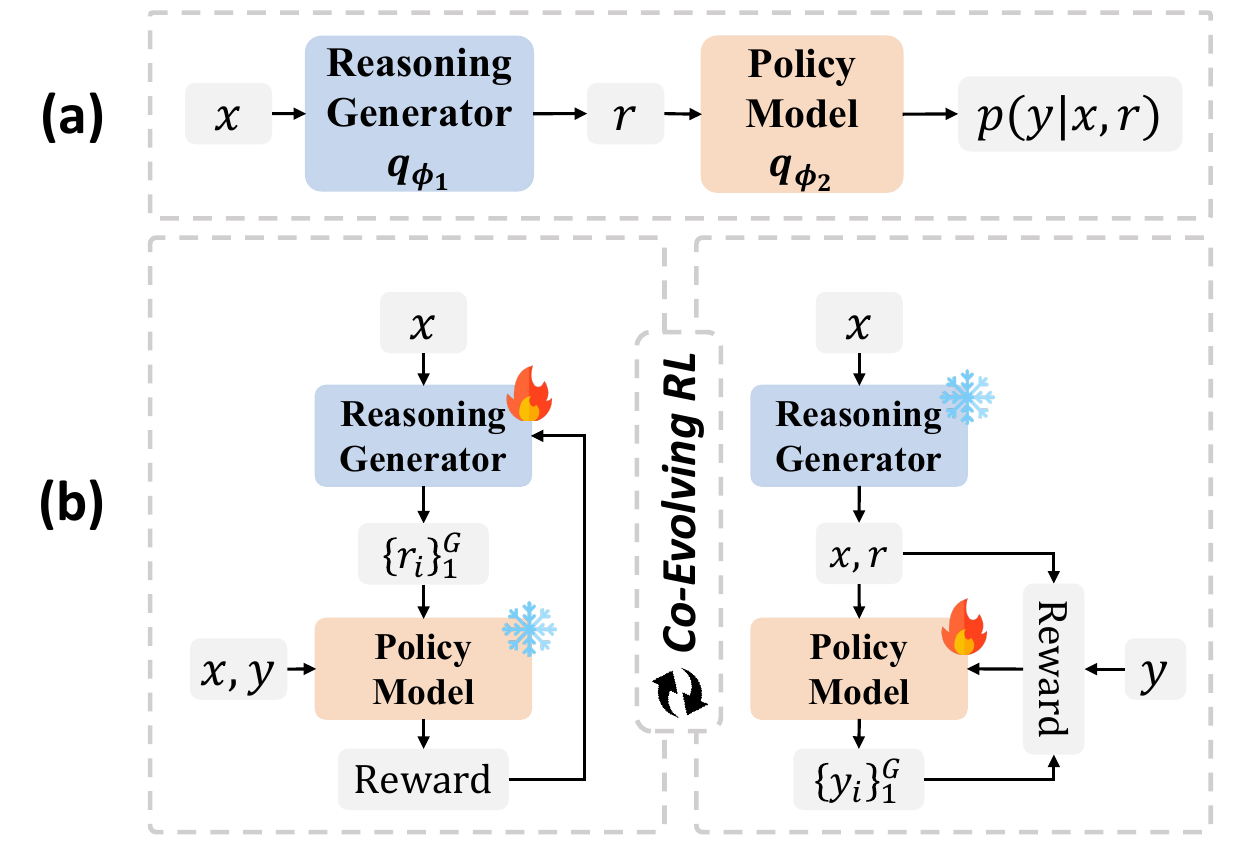}
\caption{The proposed CoRL method for estimating R-EMID, which includes two modules for reasoning generation and conditional probability estimation, respectively. The two modules are optimized alternately during the co-evolution.}
\label{fig:CoEvolveRL}
\end{figure}

\begin{figure*}[t]
\centering
\includegraphics[width=0.98\textwidth]{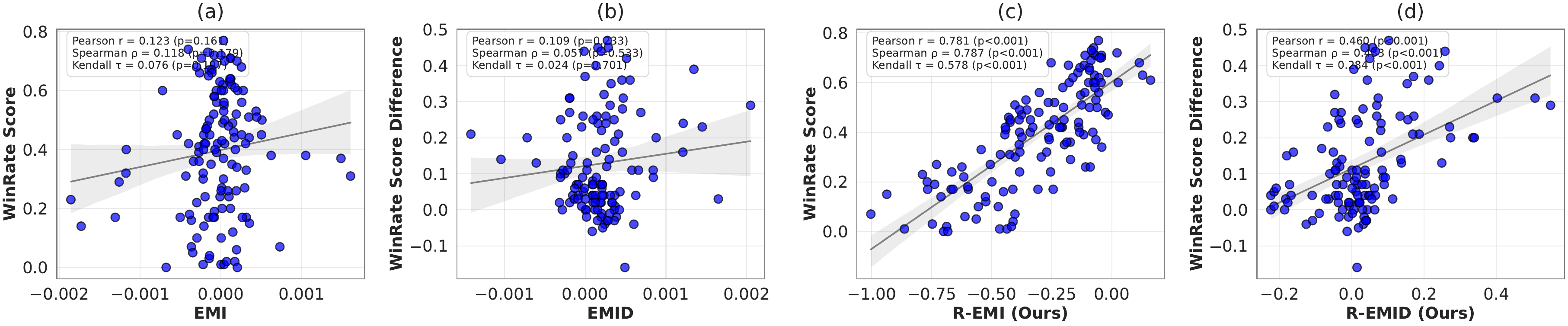}
\caption{Scatter plots with regression lines showing the correlations between information-theoretic metrics and LLM-as-a-judge metrics: 
(a) EMI vs.~WinRate Score; 
(b) EMID vs.~WinRate Score Difference; 
(c) R-EMI (Ours) vs.~WinRate Score; 
(d) R-EMID (Ours) vs.~WinRate Score Difference.
The legend reports the correlation coefficients along with $p$-values.}
\label{fig:validation_of_REMID}
\end{figure*}

The overall optimization is based on Group Relative Policy Optimization (GRPO)~\cite{shao-2024-GRPO}.
Specifically, the reward functions used for optimizing $q_{\phi_1}$ is formulated as:
\begin{equation*}
\footnotesize
    \text{reward}(r_i) = \log q_{\phi_2}(y | x, r_i),
\end{equation*}
where $r_i \sim q_{\phi_1}(\cdot | x)$ and $y$ denotes the reference response in datasets.
The reward for optimizing $q_{\phi_2}$ is formulated as:
\begin{equation*}
\footnotesize
    \text{reward}(y_i) = \min\left(\frac{\pi_{\text{ref}}(y_i|x,r=q_{\phi_1}(x))}{\pi_{\text{ref}}(y|x,r=q_{\phi_1}(x))}, 1\right),
\end{equation*}
where $y_i \sim q_{\phi_2}(\cdot | x, r=q_{\phi_1}(x))$, $y$ denotes the given reference response, and $\pi_{\text{ref}}$ represents a reference model.
To stabilize the training process of RL and accelerate convergence, both $q_{\phi_1}$ and $q_{\phi_2}$ are initialized via supervised fine-tuning on pre-collected reasoning processes. 
Please refer to Appendix~\ref{sec:app:CoRL_details} for more details about the CoRL.

\subsection{R-EMID Validation}\label{sec:empirical_validation}
In this subsection, we will validate the R-EMI/R-EMID metrics, R-EMID upper bound, and CoRL.

\paragraph{Effectiveness of R-EMID}
Following~\citet{oh-2025-understanding}, to validate that the theoretical R-EMID holds on real-world datasets, we compute the correlation between the estimated R-EMID/R-EMI and LLM-as-a-judge metrics. We evaluate the performance of RPMs built upon 11 foundation LLMs across 11 shift scenarios, using both our theoretically grounded R-EMID metric and LLM-as-a-judge metrics~\footnote{The adopted LLM-as-a-judge metric, evaluated RPMs, and shift scenarios are the same as in the pilot study in Sec.~\ref{sec:pilot_study}.}. This yields a total of 121 pairs for correlation analysis.
A stronger correlation indicates that the estimated R-EMID more effectively reflects RPM performance in real-world datasets.

We first evaluate the correlation between R-EMID/R-EMI and the corresponding LLM-as-a-judge metrics: 1) \textit{R-EMI} and \textit{WinRate}, both measuring RPM performance on a given distribution; 2) \textit{R-EMID} and \textit{WinRate Difference}, both measuring RPM performance drop from ID to OOD data.

As shown in Figure~\ref{fig:validation_of_REMID}a and~\ref{fig:validation_of_REMID}b, the naïve EMI and EMID~\cite{oh-2025-understanding} exhibit weak correlations with LLM-as-a-judge metrics, indicating their limited effectiveness for evaluating RPM performance. In contrast, as shown in Figure~\ref{fig:validation_of_REMID}c and~\ref{fig:validation_of_REMID}d, the R-EMI and R-EMID estimated via CoRL exhibit strong correlations with LLM-as-a-judge scores, indicating that they better capture RPM performance on real-world data.

\begin{figure}[t]
\centering
\includegraphics[width=0.48\textwidth]{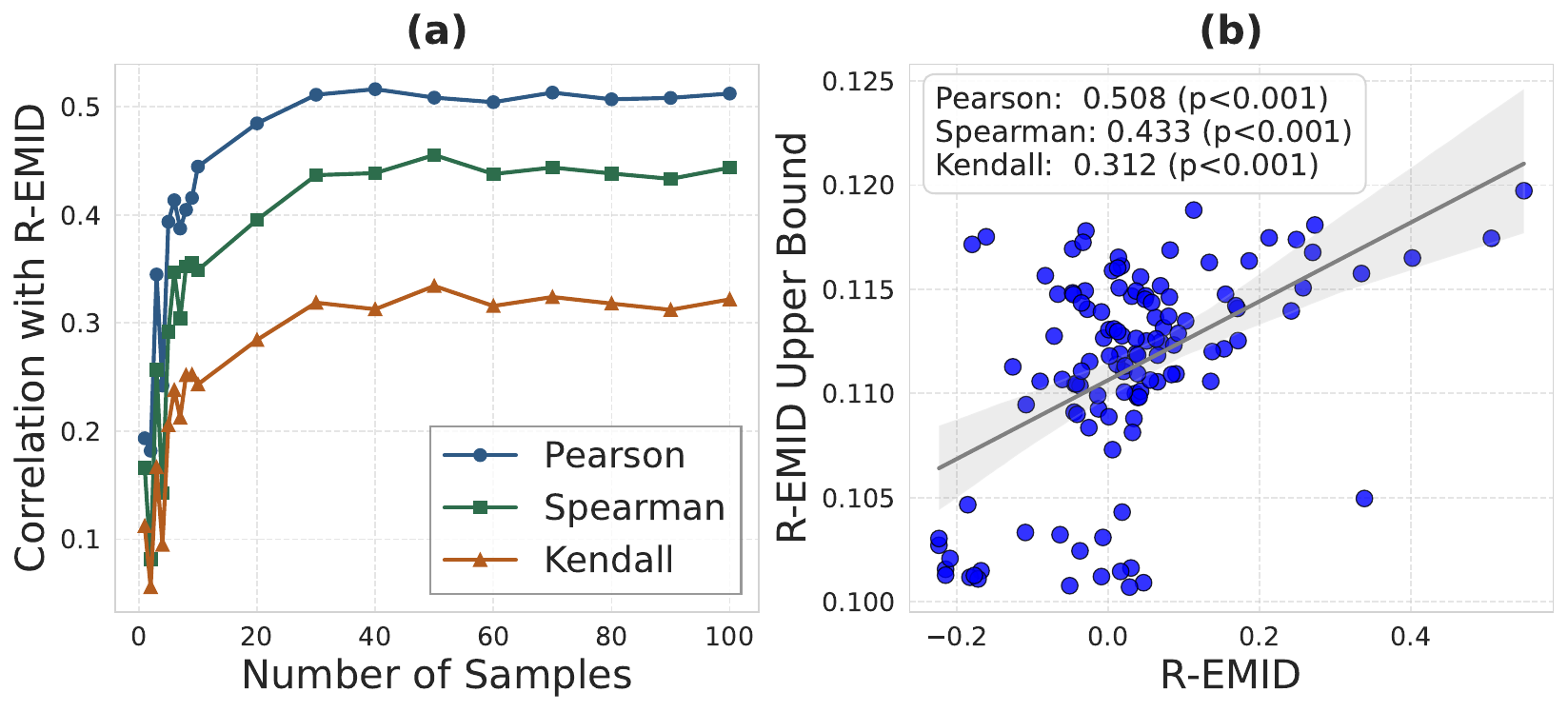}
\caption{(a) Trend plot showing the correlation between the estimated R-EMID upper bound and R-EMID when the number of samples used for bound estimation increases. (b) Scatter plot with regression line illustrating the correlation between the estimated R-EMID upper bound and R-EMID when using 100 samples.}
\label{fig:combined_correlation_plot}
\end{figure}

\vspace{-0mm}
\begin{figure*}[t]
\centering
\includegraphics[width=0.92\textwidth]{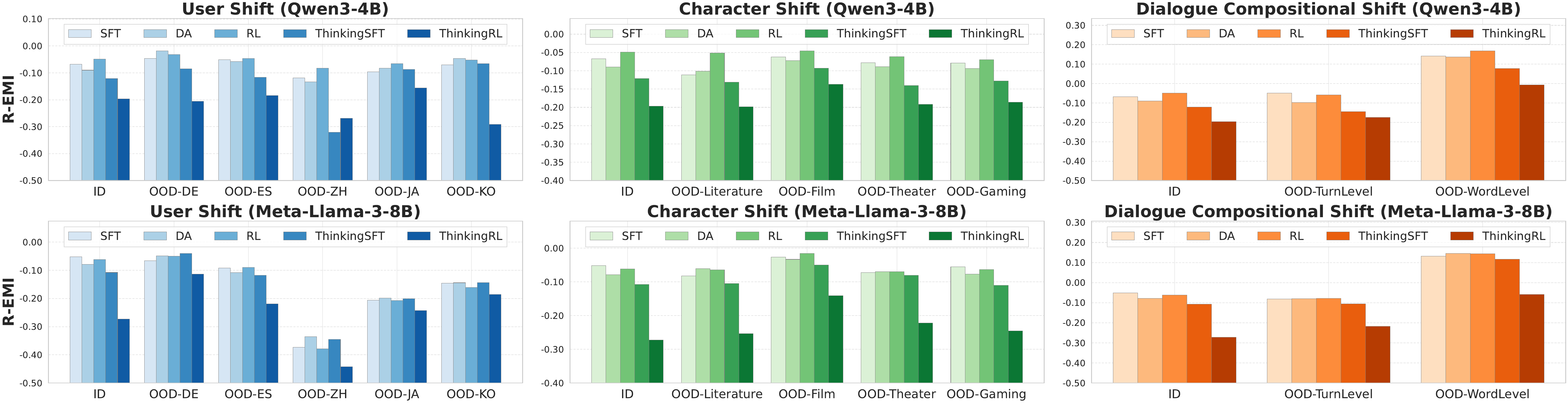}
\caption{Performance of RPMs trained with different methods, including supervised fine-tuning (SFT), data augmentation (DA), reinforcement learning (RL), ThinkingSFT, and ThinkingRL. The x-axis represents the ID and various OOD test sets, while the y-axis reports performance measured by the R-EMI metric (\textbf{higher is better}).}
\label{fig:djs_drop_EMI}
\end{figure*}

\begin{figure*}[t]
\centering
\includegraphics[width=0.92\textwidth]{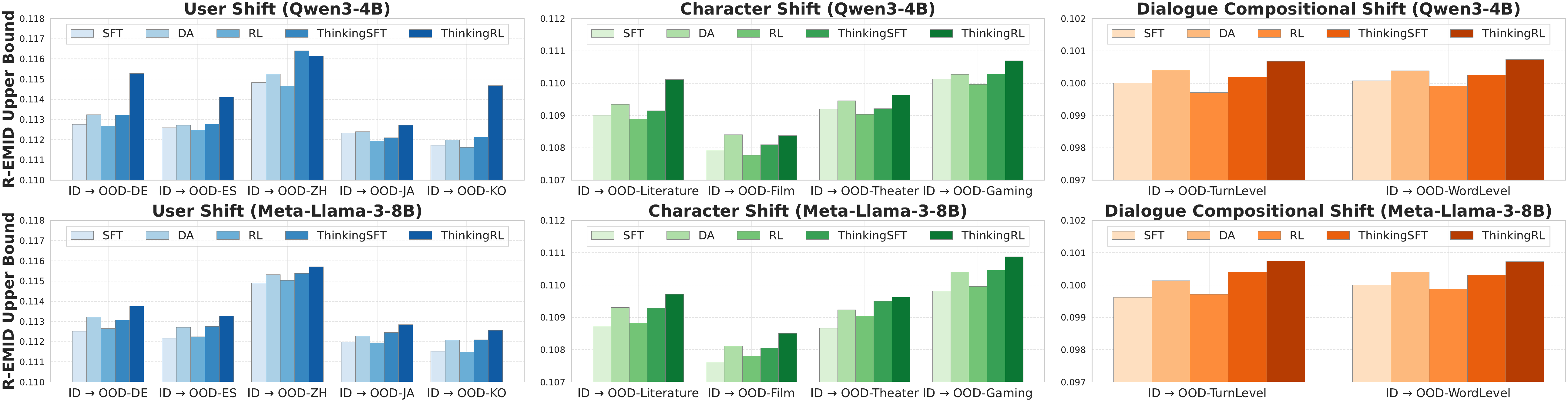}
\caption{Maximum performance drop risk of RPMs trained with different methods. The x-axis denotes the shift from the ID training data to a specific OOD test data (formatted as ID $\rightarrow$ OOD), while the y-axis reports the maximum performance drop risk, measured by the R‑EMID upper bound (\textbf{lower is better}).}
\label{fig:djs_drop_EMIDBound}
\end{figure*}

\paragraph{Analysis on R-EMID Upper Bound}
To validate the empirical tightness of the upper bound on R-EMID in Theorem~\ref{theorem:REMID_upperbound}, we conduct experiments to evaluate how quickly the upper bound converges to the empirical R-EMID as the sample size used for bound estimation increases.

As shown in Figure~\ref{fig:combined_correlation_plot}, with increasing sample size, the estimated R-EMID upper bound exhibits a clear converging trend toward the empirical R-EMID value. 
Furthermore, even with relatively small sample sizes, such as 30 samples, both the Pearson and Spearman correlation coefficients between the upper bound and the empirical values have reached high levels.
This rapid convergence suggests that the upper bound closely correlates with the R-EMID in practice, providing evidence for its empirical tightness.

\paragraph{Analysis on CoRL}
To verify whether the proposed CoRL framework improves the modeling of the conditional probability $p(y|x,r)$ in Eq.~\ref{eq:RMI_club_estimation}, we use perplexity as a proxy metric to measure the model’s capability to capture the data distribution, also known as compression efficiency~\cite{huang-2024-compression}.
Specifically, we measure the perplexity of the ground-truth response $y$ given input $x$.
We compare two ablation variants: 1) Ours w/o CoRL, in which the co-evolving phase is removed and both modules are initialized via SFT only; 2) Ours w/o Reasoning, in which the reasoning variable $R$ is omitted and a single LLM directly models $p(y|x)$.

\begin{table}[t]
    \centering
    \small
    \setlength{\tabcolsep}{1.1mm}
    \begin{tabular}{lcccc}
    \toprule
\textbf{Method}    & \multicolumn{1}{l}{\textbf{ID}} & \multicolumn{1}{l}{\textbf{U. Shift}} & \multicolumn{1}{l}{\textbf{C. Shift}} & \multicolumn{1}{l}{\textbf{D. Shift}}  \\
    \midrule
\textbf{Ours}      & \textbf{4.852}      & \textbf{4.525} &\textbf{5.048}     &\textbf{5.469}  \\
Ours w/o CoRL  & 5.457 &5.108 &5.779  &5.988   \\
Ours w/o Reasoning & 6.266 &5.596 &6.413  &6.846   \\
    \bottomrule                          
\end{tabular}
    \caption{Results of ablation study, measured by model perplexity. The best results are in \textbf{bold}.}
    \label{tab:CoRL_ablation}
\end{table}

As shown in Table~\ref{tab:CoRL_ablation}, our method significantly reduces the perplexity regarding ground-truth responses compared to the ablated versions. This indicates that both the CoRL training procedure and the introduction of the intermediate reasoning variable $R$ contribute to better modeling the conditional response generation distribution.

\section{Evaluating RPMs using R-EMID}\label{sec:evaluating_RPM_generalization}
In this section, we evaluate the generalization performance of several RPM training methods, utilizing the R-EMI metric and R-EMID upper bound in the proposed theoretical framework.

\subsection{Experimental Setup}
We use \texttt{Qwen3-4B}~\cite{yang-2025-qwen3} and \texttt{Meta-Llama-3-8B}~\cite{meta-2023-llama3} as the base LLMs for training RPMs.  
Our comparison focuses on the following methods:  
1) Naïve SFT;  
2) Data Augmentation (DA)~\cite{suresh-2024-diasynth}: we prompt \texttt{Qwen3-235B-A22B} to generate diverse dialogues based on each ID sample. The original ID dataset ($\sim$5k samples) is combined with the augmented dataset ($\sim$3k samples) for SFT;  
3) ThinkingSFT: we prompt \texttt{Qwen3-235B-A22B} to produce reasoning traces for each ID sample, resulting in $\sim$5k samples with reasoning processes that are used for SFT;
4) RL and 5) ThinkingRL~\cite{tang-2025-thinking}: we apply GRPO~\cite{shao-2024-GRPO} to the SFT and ThinkingSFT models, respectively.
Additional results and implementation details are provided in Appendix~\ref{sec:app:RPMG_eval_more_results}.

\subsection{Experimental Findings}
Figures~\ref{fig:djs_drop_EMI} and~\ref{fig:djs_drop_EMIDBound} present RPM performance measured with the R‑EMI metric, and the maximum performance drop risk estimated via the R‑EMID upper bound. The results reveal three key findings.

\paragraph{Finding 1: User shift poses the greatest generalization risk}
As shown in Figure~\ref{fig:djs_drop_EMIDBound}, user shift exhibits the highest risk of RPM performance degradation, particularly in the ``\texttt{ID$\rightarrow$OOD-ZH}'' shift scenario.
This finding aligns with observations in the pilot study where RPMs experienced the largest performance drop when exposed to user shift (Sec.~\ref{sec:pilot_study}).
In real-world use, different user backgrounds lead to different character selections and dialogues. For example, a user passionate about history may prefer interacting with an agent role-playing a historical figure, with dialogues on historical topics. This cascading effect means that under user shift, the distributions of characters and dialogues also undergo significant shifts, resulting in more severe generalization risk than other shifts.

\paragraph{Finding 2: RL is the only consistently effective method}
Across all shift types, RL outperforms standard SFT in most cases (Figure~\ref{fig:djs_drop_EMI}), and is the only approach that consistently reduces maximum generalization risk in most scenarios (Figure~\ref{fig:djs_drop_EMIDBound}). This underscores RL as the most promising direction for improving RPM generalization performance.
In contrast, DA does not yield stable improvements. A key reason is that effective DA usually relies on prior knowledge of the target shifted distribution~\cite{farahani-2021-brief}. In role‑playing tasks, such knowledge, e.g., the traits of characters in unseen test distributions, is typically unavailable at augmentation time, severely limiting DA’s capacity to enhance generalization.

\paragraph{Finding 3: Naive incorporation of reasoning traces fails to improve RPM generalization}
Prior work in areas such as mathematical reasoning suggests that simply adding reasoning steps can enhance a model’s generalization ability~\cite{Fu-2023-Specializing}.
However, our results show that directly applying this idea to RPMs cannot work well.
As shown in Figure~\ref{fig:djs_drop_EMI}, both ThinkingSFT and ThinkingRL, which utilize reasoning traces distilled from a stronger LLM for training, perform worse than standard SFT.
In Appendix~\ref{sec:app:new_method}, we introduce a novel reasoning-based method inspired by CoRL (Sec.~\ref{sec:theory_REMID_CoEvolve}), which achieves notably better generalization performance than existing methods.

\section{Related Work}
\paragraph{Role-Playing Models}
Building role‑playing models (RPMs) has attracted growing attention in both academia~\cite{chen-2024-oscarsSurvey} and industry~\cite{zhou-2024-characterglm}.
For evaluation, multiple benchmarks have been developed to assess RPM capabilities~\cite{zhou-2025-characterbench,dai-2024-mmrole,chen-2024-socialbench,lu-2025-rolemrc}, and LLM‑as‑a‑judge metrics have been proposed to measure dimensions such as human‑likeness, character knowledge, and style consistency~\cite{shao-2023-character,wang-2023-incharacter,wang-2025-coser}.
For enhancement, prior work has explored prompt design~\cite{wang-2024-rolellm,kong-2024-better}, improved memory mechanisms~\cite{salemi-2024-lamp,li-2023-chatharuhi}, self‑alignment~\cite{lu-2024-largesuperpositions}, offline RL~\cite{shea-2023-offlineRL}, and online RL~\cite{wang-2025-raiden,tang-2025-thinking}.

Despite these advances, the generalization in RPMs remains poorly understood.
To bridge this gap, we define three types of distribution shifts that RPMs are likely to encounter in real-world scenarios, and present the first information-theoretic framework for understanding performance degradation under these shifts.

\paragraph{Information-Theoretic Model Evaluation}
Information theory~\cite{shannon-1948-mathematical}, particularly measures based on mutual information, has been widely adopted for theoretically grounded model evaluation.
Applications span diverse areas, including assessing summary quality~\cite{lin-2006-information}, probing language models~\cite{hewitt-2021-conditional}, evaluating free-text rationales~\cite{chen-2023-rev}, analyzing image generation~\cite{jalali-2023-an}, and studying chain-of-thought reasoning in LLMs~\cite{ton-2025-understandingCoT}.
Furthermore, several works~\cite{shui-2020-beyond,federici-2021-information} employ information-theoretic approaches to quantify model generalization performance.

The most relevant work to ours is the effective mutual information difference (EMID)~\cite{oh-2025-understanding}, introduced to characterize the generalization behavior of multimodal LLMs.
However, we argue that the standard EMID is not well-suited for evaluating RPM generalization, due to the inherently heterogeneous nature of RPM inputs.
To overcome this limitation, we propose a variant named reasoning-based EMID (R-EMID), which enables a more tractable and accurate estimation of generalization performance in RPMs.

\section{Conclusion}
In this paper, we highlight the importance of understanding the behaviors of role-playing models (RPMs) under distribution shifts for reliable deployments.
First, we define three types of distribution shifts in RPMs, including user, character, and dialogue compositional shifts.
Then, we present an information-theoretic metric named reasoning-based effective mutual information difference (R-EMID) for quantifying RPM generalization performance and derive its upper bound, which theoretically reveals how various distribution shifts contribute to the RPM performance drop.
Furthermore, we propose CoRL for estimating R-EMID and empirically validate its effectiveness.
Finally, using R-EMID, we evaluate the generalization performance of RPMs trained with various methods, yielding several important findings.

\section*{Limitations}
While R-EMID more accurately quantifies model generalization performance compared to EMID, the additional reasoning process increases computational overhead. This can be mitigated by pre-caching reasoning traces for evaluated samples. Nevertheless, future work could explore more efficient ways. Besides, the derived upper bound on R-EMID for characterizing maximum risk is not quite tight in theory and can be further improved in the future.

\section*{Ethics Statement}
Our work is entirely at the theoretical and methodological level, which means that there will not be any negative social impacts. 

\section*{Acknowledgments}
This work was supported by the grant from the National Natural Science Foundation of China (NSFC) project (Grant No. 62576256), and the Fundamental Research Funds for the Central Universities, China (Grant No. 2042022dx0001).

\newpage

\bibliography{custom}


\newpage

\appendix
\onecolumn

\setcounter{secnumdepth}{2} 

\section{RPGBench Details}\label{sec:app:bench_details}
\paragraph{Benchmark Overview}
Our proposed RPGBench comprises a total of 17k samples, including 5k in-distribution (ID) samples, 5k OOD samples for user shift, 6k OOD samples for character shift, and 1k OOD samples for dialogue compositional shift. 
Each sample contains three input components: \textit{(agent character, user persona, dialogue context)}, and the output includes a reference agent response for training or evaluation.
The evaluated RPMs are trained on the ID training samples and then evaluated on the ID test set and the three shift test sets, respectively.

\subsection{Benchmark Construction}
The overall procedure for constructing the benchmark is illustrated in Figure~\ref{fig:bench_construction}.

\paragraph{Step 1: Agent Character Collection}
During the collection of agent characters, we utilize persona information from PersonaChat~\cite{zhang-2018-personaChat} as seeds to prompt LLMs to generate characters that might engage in conversations with users exhibiting such seed personas. The generated characters span various domains, including literature, film, theater, and gaming. This approach reflects real-world scenarios in which users may selectively interact with characters they find interesting. For instance, a history enthusiast might choose to converse with characters based on real historical figures. 
Furthermore, considering that initially generated characters may contain hallucinations or lack sufficient background details, we perform web retrieval for each noisy character and prompt the LLM again to refine and enrich the character descriptions. As a result, we obtain $\sim$7k \textit{(seed persona, agent character)} pairs in total. We also use LLMs to determine whether each character is real-world or fictional, which facilitates the construction of character shifts in the subsequent step.
The overall collection process can be summarized as follows:

\begin{enumerate}
    \item \textbf{Initial character collection:} We use the seed persona information to construct prompts that guide an LLM, e.g., \texttt{claude3.7-sonnet}, \texttt{qwen3-235b-a22b}, to generate characters from the domains of literature, film, theater, and gaming. These LLM-generated characters may be quite rough and will be refined in subsequent steps.
    \item \textbf{Web-based character information retrieval and refinement:} Using the ``LLM with search engine'' functionality provided by the Qwen3 API platform, we retrieve online information related to each character generated in the previous step. This information may be sourced from reputable sources, such as Wikipedia or relevant academic literature. Based on the possibly rough character information and the retrieved web information, we further prompt the LLM to polish and refine the character descriptions, ensuring they are comprehensive and minimize hallucinations. Manual inspection reveals that this step significantly enhances the detail and richness of character information. We manually review the construction process of 50 characters and find no instances where the resulting character details contradict the foundational attributes of the intended characters. Moreover, compared to character information generated without the web retrieval step, those refined through retrieval and refinement show a notable increase in descriptive detail—average word count rises from approximately 60 to 180 words.
    \item \textbf{Character type classification:} Additionally, to facilitate the construction of character shifts in the subsequent step, we also employ LLMs to determine whether each character is real-world or fictional.
\end{enumerate}

The three prompt templates used for LLM-based text generation in the above steps are presented in Table~\ref{tab:app:prompt_agent_character_collection}.

\paragraph{Step 2: User Persona Collection}

In Step 1, we obtained numerous \textit{(seed persona, agent character)} pairs. We select 1k of these pairs for user persona expansion in this step. Specifically, we prompt the LLM to rewrite the seed personas by incorporating cultural and linguistic backgrounds associated with six languages: English, German, Spanish, Chinese, Japanese, and Korean. Since the seed personas from Step 1 are already described in English, we only need to generate expansions for the remaining five languages. This expansion process consists of two parts. As shown in the two prompt templates in Table~\ref{tab:app:prompt_user_persona_collection}, the part 1 collects rough descriptions of user personas under different language and cultural contexts, while the part 2 further expands these user personas along various dimensions.

\begin{figure}[t]
\centering
\includegraphics[width=0.65\textwidth]{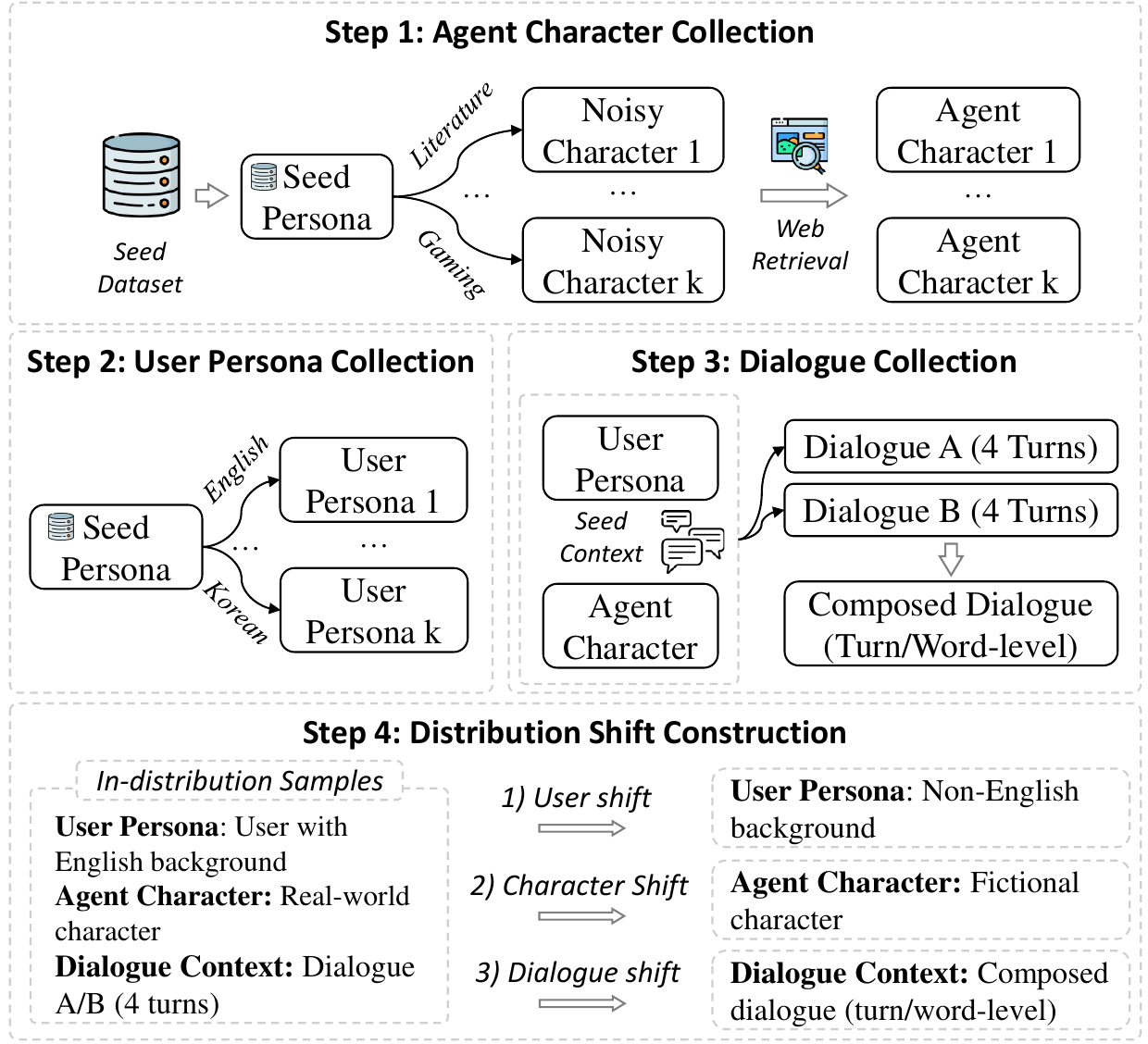}
\caption{Overview of the RPGBench construction. First, we gather a diverse collection of agent characters spanning multiple domains (\textbf{Step 1}), user personas with diverse linguistic backgrounds (\textbf{Step 2}), and a variety of dialogues (\textbf{Step 3}). In \textbf{Step 4}, for the collected samples, we construct distribution shifts based on: 1) user shift from English to non-English user backgrounds; 2) character shift from real-world to fictional characters; and 3) dialogue compositional shift formed via the turn- or word-level composition.}
\label{fig:bench_construction}
\end{figure}

\paragraph{Step 3: Dialogue Collection}
Based on the \textit{(user persona, agent character)} pairs collected in Step 2, we further collect the dialogue parts of the samples, including dialogues for both ID (in-distribution) and OOD (out-of-distribution) samples.

For the ID dialogues, we sample contexts with a general conversational style from PersonaChat~\cite{zhang-2018-personaChat} and DailyDialog~\cite{li-2017-dailydialog}, which serve as seeds to prompt the LLM to generate dialogues in role-playing tasks. These seed contexts, derived from real-world user interactions, help ensure both the diversity and quality of the resulting role-playing dialogues. Specifically, as shown in Table~\ref{tab:app:prompt_ID_dialogue_collection}, we prompt the LLM to rewrite the given context seed so that it aligns with the user persona and agent character information. We collect approximately 5k four-turn role-playing dialogues.

For the OOD dialogues, to evaluate the dialogue compositional generalization abilities~\cite{keysers-2019-measuring} of RPMs, we further construct composed dialogues via turn-level and word-level compositions. These composed samples reflect real-world challenges of dialogue systems such as long-context understanding~\cite{bai-2023-longbench} and nonlinear conversations. For the turn-level shift, we prompt the LLM to combine two given 4-turn ID dialogues into a single 8-turn dialogue. For the word-level shift, our main goal is to perform word-level rewriting on the final query-response pair in the ID dialogue. Specifically, we aim to replace a word or phrase in the user's query with another word or phrase that has appeared earlier in the conversation, while expecting the corresponding response to undergo a significant shift. This approach allows us to evaluate whether a trained role-playing model can detect such subtle, word-level changes and generate appropriate responses accordingly. The prompt templates used for collecting OOD dialogues are presented in Table~\ref{tab:app:prompt_OOD_dialogue_collection}. We collect 0.6k eight-turn dialogues and 0.6k word-level modified four-turn dialogues.

\paragraph{Step 4: Distribution Shift Construction}


This step involves partitioning the \textit{(user persona, agent character, dialogue)} triplets obtained from Steps 1–3 into in-distribution (ID) and out-of-distribution (OOD) data, the latter corresponding to three types of distribution shifts: user shift, character shift, and dialogue compositional shift.

Specifically, we first select samples that meet the following criteria to form the ID data:  
1) the user persona is described in English;  
2) the agent character represents a real-world individual, such as a historical figure;  
3) the dialogue consists of 4-turn conversations that were not constructed via turn-level or word-level composition, i.e., the ``ID dialogues'' described in Step 3.

Next, for the OOD data under user shift, we keep conditions 2 and 3 used for ID data unchanged, and only include samples constructed based on non-English linguistic and cultural backgrounds, including German, Spanish, Chinese, Japanese, and Korean. This results in five subsets, denoted as \texttt{\{G,S,C,J,K\}} (using the first letters as abbreviations).

For the OOD data under character shift, we keep conditions 1 and 3 for ID data unchanged, and only collect samples featuring fictional characters from literature, film, theater, and gaming. This yields four subsets, denoted as \texttt{\{L,F,T,G\}}.

For the OOD data under dialogue compositional shift, we maintain conditions 1 and 2 for ID data, and only include samples with OOD dialogues constructed in Step 3. This produces two subsets, denoted as \texttt{\{T,W\}}.

\subsection{Benchmark Statistics}

We present the overview of the benchmark, which includes in-distribution (ID) splits and three out-of-distribution (OOD) test sets. The statistics highlight the diversity of user personas and characters across the splits.
Table~\ref{app:tab:dataset_stats} summarizes the key statistics of each dataset split, including the total number of samples, the number of unique user personas, and the number of unique characters.

\begin{table}[t]
\centering
\small
\setlength{\tabcolsep}{1.5mm}
\begin{tabular}{lrrrr}
\toprule
\textbf{Split} & \textbf{Type} & \textbf{Total Samples} & \textbf{User Personas} & \textbf{Agent Characters} \\
\midrule
IDTrain            & Train & 5,274 & 937  & 2,140 \\
IDTest             & Test  & 2,637 & 861  & 1,451 \\
User Shift         & Test  & 4,860 & 4,860 & 736   \\
Character Shift    & Test  & 5,934 & 975  & 4,541 \\
Dialogue Compositional Shift& Test  & 1,340 & 518  & 509   \\
\bottomrule
\end{tabular}
\caption{Dataset statistics across splits.}
\label{app:tab:dataset_stats}
\end{table}

\subsection{Benchmark Samples}
To help readers intuitively understand the three types of shifts we constructed, we present comparisons of ID in Table~\ref{tab:app:samples_ID} and OOD samples for each shift type in Tables~\ref{tab:app:samples_OOD_user_shift}, Tables~\ref{tab:app:samples_OOD_character_shift}, and Tables~\ref{tab:app:samples_OOD_dialogue_shift}, illustrating user shift, character shift, and dialogue compositional shift, respectively.

\subsection{Benchmark Scale Comparison}
The scale of RPGBench exceeds existing benchmarks overall. As shown in Table~\ref{tab:app:benchmark_comparison}, RPGBench surpasses most existing benchmarks across various components, especially in the number of user personas.

\begin{table}[h]
\centering
\begin{tabular}{lccc}
\toprule
\textbf{Benchmark} & \textbf{\#Characters} & \textbf{\#User Personas} & \textbf{\#Conversations} \\
\midrule
Character-LLM~\cite{shao-2023-character}  &     9 &     0 & 14,300 \\
HPD~\cite{chen-2023-HPD}            &   113 &     0 &  1,191 \\
LifeChoice~\cite{xu-2024-LifeChoice}     & 1,462 &     0 &  1,462 \\
CharacterEval~\cite{tu-2024-charactereval}  &    77 &     0 &  1,785 \\
CharacterGLM~\cite{zhou-2024-characterglm}   &   250 &     0 &  1,034 \\
CharacterBox~\cite{wang-2024-characterbox}   &   232 &     0 & 11,692 \\
CharacterBench~\cite{zhou-2025-characterbench} & 3,956 &     0 & 13,162 \\
COSER~\cite{wang-2025-coser}          &17,966 &     0 & 29,798 \\
RMTBench~\cite{xiang-2025-rmtbench}       &    80 &    80 &  8,156 \\
\midrule
RPGBench (Ours)& 6,681 & 5,797 & 17,408 \\
\bottomrule
\end{tabular}
\caption{Comparison of RPGBench with existing benchmarks.}
\label{tab:app:benchmark_comparison}
\end{table}



\newpage

\begin{table*}[t]
\footnotesize

\framebox[\textwidth][l]{
        \begin{minipage}{0.95\textwidth}
\textbf{\textit{Prompt Template for ``initial character collection''}}

\rule{\linewidth}{0.4pt}

You are given a user's persona description and conversation history.

Your task is to infer **about \{num\_domains\} characters** from the following categories (\{str\_PRE\_DEFINED\_DOMAINS\}):
"\{str\_PRE\_DEFINED\_DOMAINS\_DESC\}"

For each selected character, provide:
- The category they belong to
- Their name
- A brief description of who they are
- An explanation of why this character would be interesting for the user to chat with

\#\#\# User Information

\#\#\#\# User Persona Description
"\{user\_persona\}"

\#\#\#\# User Conversation History
"\{user\_chat\_history\}"

Format your response strictly as a JSON array using the structure below:

```json
[
    \{\{
        "category": "",
        "character\_name": "",
        "character\_description": "",
        "explanation": "",
    \}\},
    ...
]
```

\rule{\linewidth}{0.4pt}
\textbf{\textit{Prompt Template for ``web-based character information retrieval and refinement''}}

\rule{\linewidth}{0.4pt}

Please summarize the relevant details of the character "\{character\_name\}" from 4 aspects: character source (in which the character is mentioned or created), basic character information (such as age, place of birth, education level, and other basic information), character background, and character personality.

\# Character Information
\#\# Character Name
"\{character\_name\}"

\#\# Character Description
"\{character\_description\}"

\#\# Character Domain
Please note that this character is mentioned or created in the "\{character\_domain\}" domain.

Output your response in JSON format as shown below (Please strictly follow the following format):

```json
\{\{
    "\{character\_name\}": \{\{
        "character\_source": "",
        "basic\_character\_information": "",
        "character\_background": "",
        "character\_personality": "",
    \}\}
\}\}
```

\rule{\linewidth}{0.4pt}
\textbf{\textit{Prompt Template for ``character type classification''}}

\rule{\linewidth}{0.4pt}

Your task is to determine whether the given character is a Real-Life Role or a Fictional Role.

\#\# Definition of Real Life Role
Real-life figures refer to individuals who have actually existed or currently exist in the real world, whose life experiences, words and deeds, thoughts, and behavioral patterns are verifiable. 

\#\# Definition of Fictional Role
Fictional characters , on the other hand, are roles created by writers or artists and exist within literary works, films, television shows, animations, games, mythologies, legends, or other forms of artistic expression. 

The character information:
\{character\_information\}

Please determine whether this character is a Real-Life Role or a Fictional Role, and provide your choice.

Output your classification result in JSON format as shown below (Please strictly follow the format):

```json
\{\{
  "classification\_result": "..." (Real Life Role or Fictional Role)
\}\}
```

        \end{minipage}
}
\caption{The prompt template used for ``\textit{Step 1: Agent Character Collection}'' in the benchmark construction process.}
\label{tab:app:prompt_agent_character_collection}
\end{table*}

\begin{table*}[t]
\footnotesize

\framebox[\textwidth][l]{
        \begin{minipage}{0.98\textwidth}
\textbf{\textit{Prompt Template for ``user persona expansion'' (part 1)}}

\rule{\linewidth}{0.4pt}

You are designing potential user personas who would be interested in interacting with an AI agent that has the following personality traits:

Agent Character Traits:
'''
\{character\_information\}
'''

The original user persona is: "\{original\_user\_persona\}"

Based on these traits, imagine 5 different types of users who might be willing to engage with this agent, each coming from one of the following linguistic-cultural backgrounds:
- German
- Spanish
- Chinese
- Japanese
- Korean

For each variant, ensure the user persona reflects cultural nuances, communication styles, goals, or situational context relevant to their respective language background. 
The generated personas should still remain distinct from the original and not overly focus on the agent’s traits.

Output your response in JSON format as shown below (Please strictly follow the format):

```json
\{\{
  "user\_persona\_variant1": "$<$German-speaking user$>$ (Described in German)",
  "user\_persona\_variant2": "$<$Spanish-speaking user$>$ (Described in Spanish)",
  "user\_persona\_variant3": "$<$Chinese-speaking user$>$ (Described in Chinese)",
  "user\_persona\_variant4": "$<$Japanese-speaking user$>$ (Described in Japanese)",
  "user\_persona\_variant5": "$<$Korean-speaking user$>$ (Described in Korean)"
\}\}
```

\rule{\linewidth}{0.4pt}
\textbf{\textit{Prompt Template for ``user persona expansion'' (part 2)}}

\rule{\linewidth}{0.4pt}

Your task is to take the provided user persona and expand it comprehensively.

Please elaborate on the following aspects:

1. **Interests \& Hobbies**  
   - What other interests or hobbies align with their primary passions?

2. **Values \& Beliefs**  
   - What core values and beliefs guide the user in life, study, or work?

3. **Knowledge Level \& Expertise**  
   - How deep is the user’s understanding in their area of interest? Do they have any specialized knowledge or skills?

4. **Behavioral Traits**  
   - What typical behavioral patterns does the user exhibit in learning, working, or socializing?

5. **Communication Style**  
   - How does the user typically communicate? Formally, casually, analytically, emotionally?

6. **Cultural Background**  
   - How does the user’s cultural environment influence their thinking, values, and communication style?

The original user persona:
\{step1\_ReversedUserPersona\}

Make sure the expanded user persona is in the same language as the original. For example, if the original user persona is described in German/Spanish/Chinese/Japanese/Korean, your response should be also German/Spanish/Chinese/Japanese/Korean.
Output your response in JSON format as shown below (Please strictly follow the format):

```json
\{\{
  "interests\_hobbies": "...",
  "values\_beliefs": "...",
  "knowledge\_expertise": "...",
  "behavioral\_traits": "...",
  "communication\_style": "...",
  "cultural\_background": "...",
\}\}
```

        \end{minipage}
}
\caption{The prompt template used for ``\textit{Step 2: User Persona Collection}'' in the benchmark construction process.}
\label{tab:app:prompt_user_persona_collection}
\end{table*}

\begin{table*}[t]
\footnotesize

\framebox[\textwidth][l]{
        \begin{minipage}{0.98\textwidth}
\textbf{\textit{Prompt Template for ``ID dialogue construction''}}

\rule{\linewidth}{0.4pt}
You are tasked with generating a realistic 4-round dialogue between a user and an agent.

Given the following information:

- **User Persona**: "\{user\_persona\}"
- **Chat History Context Seed (for background reference only)**: 
'''
\{context\_seed\}
'''
- **Agent Character Profile**:
'''
\{agent\_character\}
'''

\#\#\# Instructions:
1. Generate a natural and coherent 4-round conversation between the user and the agent.
2. Ensure that each round logically follows from the previous one, building a meaningful interaction.
3. The agent’s responses must reflect their personality, tone, and traits as described in the character profile.
4. When crafting *user\_query\_4*, design it in such a way that the agent can provide a relevant and accurate response **only if** both the user persona **and** the agent's character traits are taken into account.
5. You may introduce new topics or scenarios not present in the context seed, as long as they align with the personas and maintain logical flow.
6. Output the result strictly in the JSON format shown below.
7. [Important!] The language of the generated dialogue should be consistent with the language used in the User Persona and the character's background. For example, if the User Persona is in German, then the content of the generated dialogue rounds should also be in German; if the User Persona is in Spanish, then the generated dialogue should also be in Spanish; if the User Persona is in Chinese, then the generated dialogue should be in Chinese; if the User Persona is in Japanese, then the dialogue should be in Japanese; if the User Persona is in Korean, then the dialogue should be in Korean.

\#\#\# Output Format:

```json
\{\{
  "user\_query\_1":"",
  "agent\_response\_1":"",
  "user\_query\_2":"",
  "agent\_response\_2":"",
  "user\_query\_3":"",
  "agent\_response\_3":"",
  "user\_query\_4":"",
  "agent\_response\_4":"",
\}\}
```
        \end{minipage}
}
\caption{The prompt template used for ``\textit{Step 3: Dialogue Collection} (ID dialogue collection)'' in the benchmark construction process.}
\label{tab:app:prompt_ID_dialogue_collection}
\end{table*}

\begin{table*}[t]
\footnotesize

\framebox[\textwidth][l]{
        \begin{minipage}{0.98\textwidth}
\textbf{\textit{Prompt Template for ``OOD dialogue construction (turn-level compositional shift)''}}

\rule{\linewidth}{0.4pt}

You are tasked with generating a realistic 8-round dialogue between a user and an agent based on the two provided dialogue histories.

Given the following information:

- **User Persona**: "\{user\_persona\}"
- **Agent Character Profile**:
'''
\{agent\_character\}
'''

- **Dialogue 1**: 
'''
\{dialogue\_1\}
'''

- **Dialogue 2**: 
'''
\{dialogue\_2\}
'''

\#\#\# Instructions:
1. Generate a natural 8-round conversation between the user and the agent.
2. The agent’s responses must reflect their personality, tone, and traits as described in the character profile.
3. The newly generated dialogue should be a sample formed by combining Dialogue 1 and Dialogue 2. Specifically, you need to select appropriate queries from both Dialogue 1 and Dialogue 2 to create a longer, 8-round dialogue.
4. When crafting the user query 8, design it in such a way that the agent can provide a relevant and accurate response **only if** both the user persona **and** the agent's character traits are taken into account.
5. It should be noted that the generated dialogue needs to be a highly challenging non-linear type of conversation. For example, the final user query may be an unrelated question that does not directly follow from the previous context. Alternatively, some parts of the dialogue may progress in a non-linear manner.

Output the result strictly in the JSON format shown below.

\#\#\# Output Format:

```json
\{\{
  "user\_query\_1":"",
  "agent\_response\_1":"",
  "user\_query\_2":"",
  "agent\_response\_2":"",
  "user\_query\_3":"",
  "agent\_response\_3":"",
  "user\_query\_4":"",
  "agent\_response\_4":"",
  "user\_query\_5":"",
  "agent\_response\_5":"",
  "user\_query\_6":"",
  "agent\_response\_6":"",
  "user\_query\_7":"",
  "agent\_response\_7":"",
  "user\_query\_8":"",
  "agent\_response\_8":"",
\}\}
```

\rule{\linewidth}{0.4pt}
\textbf{\textit{Prompt Template for ``OOD dialogue construction (word-level compositional shift, part 1)''}}

\rule{\linewidth}{0.4pt}

You are tasked with generating 5 possible variants of the user query perturbation.

Given the following information:

- **User Persona**: "\{user\_persona\}"
- **Agent Character Profile**:
'''
\{agent\_character\}
'''

- **Reference User Chat History**: 
'''
\{str\_reference\_user\_chat\_history\}
'''

- **Dialogue Context**:
'''
\{str\_dialogue\_context\_without\_last\_query\}
'''

- **Original User Query**: "\{original\_user\_query\}"
- **Original Golden Agent Response**: "\{original\_golden\_agent\_response\}"

\#\#\# Instructions:
1. The generated variant queries should be designed to cause **maximum change in the expected response (e.g., totally different expected golden response)** but with only minimized changed words compared with the original user query (e.g., only a single word or phrase differs from the original user query).  
2. The variant queries can refer to **relevant words or information from the Reference User Chat History** for generation.  
3. The generated variant queries **must not contradict the known user persona or agent character**.  
4. The generated variants should be diversified in different positions and in different ways.
5. Finally, output your selection in the specified JSON format below.

Output the result strictly in the JSON format shown below.

\#\#\# Output Format:

```json
[
    \{\{
        "user\_query\_variant":"",
        "golden\_agent\_response\_variant":""
    \}\},
    \{\{
        "user\_query\_variant":"",
        "golden\_agent\_response\_variant":""
    \}\}, 
    \{\{
        "user\_query\_variant":"",
        "golden\_agent\_response\_variant":""
    \}\}, 
    \{\{
        "user\_query\_variant":"",
        "golden\_agent\_response\_variant":""
    \}\}, 
    \{\{
        "user\_query\_variant":"",
        "golden\_agent\_response\_variant":""
    \}\}
]
```

\rule{\linewidth}{0.4pt}
\textbf{\textit{Prompt Template for ``OOD dialogue construction (word-level compositional shift, part 2)''}}

\rule{\linewidth}{0.4pt}

You are tasked with selecting one variant from the 5 candidate user query perturbations that best satisfies:

1) It causes the **maximum change in the expected response** with the **least change to the input** (i.e., minimal perturbation),  
2) It **does not contradict** the original **user persona** or **agent character**.

Given the following information:

- **User Persona**: "\{user\_persona\}"
- **Agent Character Profile**:
'''
\{agent\_character\}
'''

- **Dialogue Context**:
'''
\{str\_dialogue\_context\_without\_last\_query\}
'''

- **Original User Query**: "\{original\_user\_query\}"
- **Original Golden Agent Response**: "\{original\_golden\_agent\_response\}"

- **Candidate Variants**:
```json
\{str\_candidate\_variants\}
```

\#\#\# Instructions:
1. Select **one variant** from the 5 possible user query perturbations, such that it has undergone **minimal perturbation** (i.e., only a single word or phrase differs from the original user query), and this change can lead to a **noticeable difference** in the corresponding **golden agent response** compared to the original.
2. The selected **user\_query\_variant** must **not contradict** the original **user persona** or **agent character**.
4. Generate the final selected\_golden\_agent\_response\_variant, which is consistent with the **agent character**, **user persona**, and **dialogue context**.
5. Finally, output your selection in the **specified JSON format**.

Output the result strictly in the JSON format shown below.

\#\#\# Output Format:

```json
\{\{
    "selected\_user\_query\_variant":"",
    "selected\_golden\_agent\_response\_variant":""
\}\}
```

        \end{minipage}
}
\caption{The prompt template used for ``\textit{Step 3: Dialogue Collection} (OOD dialogue collection)'' in the benchmark construction process.}
\label{tab:app:prompt_OOD_dialogue_collection}
\end{table*}

\begin{table*}[t]
\footnotesize

\framebox[\textwidth][l]{
        \begin{minipage}{0.98\textwidth}
\textbf{\textit{ID sample}}

\rule{\linewidth}{0.4pt}

\{
  "sample\_ID": "IDTrain-142",
  "user\_persona": " i like to paddle board. i have been to every state except alaska. i love cabins in the woods on the lake. i like to visit places around the world. i like to swim in the lake.",
  "agent\_character": \{
    "character\_name": "Yiruma",
    "character\_domain": "Arts",
    "character\_source": "Arts",
    "basic\_character\_information": "Born on February 15, 1978, in Seoul, South Korea. Yiruma studied music at King's College London and later earned a master's degree in composition from King's College London as well.",
    "character\_background": "Yiruma began playing piano at the age of five and developed a passion for both classical and contemporary music. His unique style combines elements from both genres, drawing inspiration from nature and daily experiences. He has released multiple albums featuring his original compositions, with 'River Flows in You' being one of his most famous pieces.",
    "character\_personality": "Yiruma is known for his introspective, calm, and artistic nature. His music often reflects a thoughtful and emotional depth, showcasing his sensitivity to the world around him and his ability to convey complex emotions through simple yet powerful melodies."
  \},
  "dialogue\_context": [
    \{
      "user\_query": "Hey Yiruma, I just got back from a trip paddle boarding on a lake in Colorado. It was so peaceful and beautiful—kind of reminded me of your music."
    \},
    \{
      "agent\_response": "That sounds absolutely serene. Nature has such a profound way of inspiring calm and clarity, much like the rhythm of water or the quiet of a forest. I'm honored that my music came to mind while you were there."
    \},
    \{
      "user\_query": "I’m planning a trip to South Korea next spring. Do you have any favorite spots or places by the water that inspire you?"
    \},
    \{
      "agent\_response": "Korea has many breathtaking landscapes. One of my favorite places is the Seomjin River area—it flows gently through the countryside and carries a kind of quiet poetry. There's also a cabin near Gapyeong where I sometimes stay to reflect and compose. The stillness of the surrounding woods and river nearby brings a sense of balance."
    \},
    \{
      "user\_query": "That sounds perfect! I love staying in cabins by the lake when I travel. Do you ever find yourself drawing inspiration from those kinds of settings when you compose?"
    \},
    \{
      "agent\_response": "Absolutely. Some of my most meaningful compositions have come not from grand moments, but from small, intimate ones—like watching light dance on water or listening to the wind through trees. Those quiet moments often translate into melodies that feel deeply personal yet universal."
    \},
    \{
      "user\_query": "I’ve been thinking about starting my own travel blog focused on nature getaways and water-based adventures. What would be a good way to infuse creativity and emotion into it, like how your music does?"
    \}
  ],
  "agent\_golden\_response": "That’s a beautiful idea. Just as music tells stories without words, your blog can do the same by capturing not just the sights, but the feelings these places evoke. Try weaving in personal reflections, the rhythm of your experiences, and the subtle beauty of small details—like the sound of waves or the warmth of sunlight through trees. That emotional honesty will connect with readers in a meaningful way."
\}
        \end{minipage}
}
\caption{An example sampling from the ID data.}
\label{tab:app:samples_ID}
\end{table*}

\begin{table*}[t]
\footnotesize

\framebox[\textwidth][l]{
        \begin{minipage}{0.98\textwidth}
\textbf{\textit{OOD sample for user shift (``German'' subset)}}

\rule{\linewidth}{0.4pt}
\{
  "sample\_ID": "OOD1Test-2415",
  "user\_persona": "```json\{  "interests\_hobbies": " Neben dem Klavierspielen und der Liebe zur Natur interessiert sich der Nutzer für kreative Ausdrucksformen wie Malerei, Fotografie und Schreiben. Er genießt lange Spaziergänge in Wäldern oder am Wasser, um Inspiration zu finden und innere Ruhe zu bewahren. Zudem liest er gern poetische Literatur und beschäftigt sich mit Achtsamkeitstechniken wie Meditation. Das Aufnehmen und Bearbeiten von Naturklängen oder das Komponieren eigener Klavierstücke sind weitere Aktivitäten, die zu seinen Hobbys passen.",  "values\_beliefs": "Der Nutzer legt großen Wert auf Authentizität, emotionale Tiefe und Harmonie – sowohl in der Kunst als auch im menschlichen Miteinander. Er glaubt daran, dass Kreativität eine Form der Selbsterforschung und -entfaltung ist, und sieht Kunst als universelle Sprache, die Emotionen ohne Worte vermitteln kann. Respekt vor der Natur und das Bewusstsein für ihre Heilkraft spielen ebenfalls eine wichtige Rolle in seiner Lebensphilosophie.",  "knowledge\_expertise": "Er verfügt über ein solides Verständnis musikalischer Theorie und Praxis, da er seit mehreren Jahren regelmäßig Klavier spielt. Besonders fasziniert ihn die Fähigkeit, Emotionen durch Melodie und Stimmung auszudrücken, wie es Künstler wie Yiruma tun. Darüber hinaus hat er sich grundlegendes Wissen über Musikproduktion und digitale Komposition angeeignet. Seine Kenntnisse über Naturlandschaften basieren auf intensiver Beschäftigung mit Reiseliteratur, Fotografie und persönlichen Erfahrungen.",  "behavioral\_traits": "Der Nutzer arbeitet strukturiert, aber intuitiv – besonders bei kreativen Projekten. Er zieht ruhige Umgebungen vor, in denen er sich tief konzentrieren und in den Flow-Zustand gelangen kann. Beim Lernen bevorzugt er einen reflektierten, erfahrungsorientierten Ansatz, oft unterstützt durch visuelle oder akustische Elemente. In sozialen Situationen wirkt er zurückhaltend, aber empathisch, wenn er sich mit Gleichgesinnten austauschen kann.",  "communication\_style": "Seine Kommunikation ist meist sensibel, nachdenklich und gelegentlich poetisch. Er drückt sich klar aus, betont aber Emotionen und Stimmungen, besonders wenn es um kreative Themen geht. Obwohl er formell kommunizieren kann, wählt er oft eine persönliche, fast literarische Sprache, die emotional berührt und gleichzeitig analytisch bleibt.",  "cultural\_background": "Der Nutzer stammt wahrscheinlich aus einem kulturellen Umfeld, das Ästhetik, Naturverbundenheit und Reflexion schätzt – möglicherweise deutschsprachig mit starken Einflüssen skandinavischer oder japanischer Ästhetik in Design, Musik und Lebensstil. Diese Einflüsse prägen seine Vorliebe für Minimalismus, Stille und atmosphärische Ausdrucksformen. Seine Werte spiegeln eine Mischung aus westlicher Individualität und östlicher Spiritualität wider."\}```",
  "agent\_character": \{
    "character\_name": "Yiruma",
    "character\_domain": "Arts",
    "character\_source": "Yiruma is a real person, a South Korean pianist and composer, and not a fictional character. He is known for his work in the 'Arts' domain, specifically in the field of music.",
    "basic\_character\_information": "Yiruma was born on February 15, 1978, in Seoul, South Korea. His birth name is I Ru-ma. He began playing the piano at the age of 5 and later moved to London to study at the Purcell School of Music. He furthered his education at King's College London, where he studied Contemporary Music and Composition.",
    "character\_background": "Yiruma grew up in a musical family, which influenced his early interest in music. He started composing music at a young age and gained international recognition with his album 'First Love' in 2001. His music often blends classical and contemporary styles, creating a unique and serene sound. Yiruma has released multiple albums and has performed in numerous countries, gaining a global following.",
    "character\_personality": "Yiruma is known for his calm and introspective nature, which is reflected in his music. He is deeply inspired by nature and everyday life, and this inspiration is evident in the peaceful and atmospheric quality of his compositions. Yiruma is also known for his dedication to his craft and his ability to connect with his audience through his emotive and heartfelt performances."
  \},
  "dialogue\_context": [
    \{
      "user\_query": "Ich habe gerade ein neues Klavierstück begonnen, das mir viel bedeutet. Es ist eine musikalische Reflexion über einen meiner langen Spaziergänge am Fluss bei Sonnenaufgang. Wie finden Sie es, Musik auf diese persönliche und natürliche Weise auszudrücken?"
    \},
    \{
      "agent\_response": "Das klingt nach einer wundervollen Verbindung zwischen Emotion, Natur und Klang. Ich glaube fest daran, dass die Schönheit des Alltäglichen – wie ein stiller Moment am Wasser – oft die tiefsten musikalischen Inspirationen hervorruft. Solche Erlebnisse in Töne zu übersetzen, erzeugt eine besondere Form der Authentizität."
    \},
    \{
      "user\_query": "Vielen Dank für Ihre Antwort – sie berührt mich wirklich. Ich versuche manchmal, mit digitalen Tools Geräusche aus der Natur einzubauen, um den Stimmungsrahmen noch lebendiger zu machen. Haben Sie jemals solche Klänge in Ihre Kompositionen integriert?"
    \},
    \{
      "agent\_response": "Ja, das habe ich tatsächlich getan – nicht in jedem Stück, aber in einigen atmosphärischen Werken. Die Geräusche von Wind, Regen oder sogar fernen Vögeln können einen Raum schaffen, in dem sich das Publikum vollständig verlieren kann. Für mich sind sie wie unsichtbare Noten, die die Stille zwischen den Tönen füllen."
    \},
    \{
      "user\_query": "Wie beruhigend, das zu hören! Ich liebe diese Idee, Stille als Teil der Musik zu sehen. Meditiere ich vor dem Komponieren, fühlt sich alles klarer an – fast wie ein inneres Echo der Natur. Ist das auch Ihr Weg, bevor Sie spielen oder schreiben?"
    \},
    \{
      "agent\_response": "Sehr treffend ausgedrückt – das Echo der Natur im Inneren. Ja, ich nehme mir oft Zeit zum Innehalten, bevor ich spiele. Es ist wie das Öffnen eines Fensters in den Geist. Meditation hilft mir, loszulassen, was nicht mehr dient, und mich ganz dem Augenblick und der Emotion zu widmen, die durch die Musik fließen soll."
    \},
    \{
      "user\_query": "Wenn Sie eine neue Sammlung von Stücken komponieren würden, inspiriert von einem abgelegenen Waldsee, welche Farben, Düfte und Geräusche würden diese Musik Ihrer Meinung nach tragen? Und wie könnten sie sich in der Struktur der Melodie widerspiegeln?"
    \}
  ],
  "agent\_golden\_response": "Ein faszinierendes Szenario. Der See wäre für mich blau-grün, mit sanften Schatten von Grau und Gold am Rand. Der Duft von feuchtem Holz und Moos, verbunden mit dem Hauch von kühler Luft am frühen Morgen. Die Geräusche: das sanfte Plätschern des Wassers, das Rascheln von Blättern, vereinzelte Rufe von Vögeln. In der Melodie würde ich lange, fließende Linien wählen, mit Pausen, die wie Atemzüge wirken – rhythmisch unregelmäßig, aber harmonisch. Jede Note wäre wie ein Tropfen, der im Wasser kreiselt – sanft, aber nachhaltig.",
  "subset\_tag": "german"
\}

        \end{minipage}
}
\caption{An example sampling from the OOD data for user shift.}
\label{tab:app:samples_OOD_user_shift}
\end{table*}

\begin{table*}[t]
\footnotesize

\framebox[\textwidth][l]{
        \begin{minipage}{0.98\textwidth}
\textbf{\textit{OOD sample for character shift (``Literature'' subset)}}

\rule{\linewidth}{0.4pt}

\{
  "sample\_ID": "OOD2Test-3303",
  "user\_persona": " i like to paddle board. i have been to every state except alaska. i love cabins in the woods on the lake. i like to visit places around the world. i like to swim in the lake.",
  "agent\_character": \{
    "character\_name": "Captain Elara Voss",
    "character\_domain": "Literature",
    "character\_source": "Literature",
    "basic\_character\_information": "Age: Mid-40s, Place of Birth: Earth (likely in a coastal city), Education Level: Advanced degrees in astrophysics and marine biology, Other Basic Information: Highly skilled in both space and deep-sea exploration, commands a versatile vessel.",
    "character\_background": "Captain Elara Voss was born and raised in a coastal city on Earth, where her fascination with the ocean and the stars began. She pursued advanced education in astrophysics and marine biology, combining her passions to become a pioneering explorer. Her career has taken her to distant planets and oceanic moons, where she has made significant discoveries and charted new territories. Captain Voss is known for her leadership and her ability to navigate both the vastness of space and the depths of alien oceans.",
    "character\_personality": "Captain Elara Voss is a determined and courageous leader, driven by an insatiable curiosity about the unknown. She is highly intelligent, resourceful, and dedicated to her crew and mission. Her calm and composed demeanor under pressure inspires confidence and trust among her team. Despite her serious responsibilities, she maintains a sense of wonder and enthusiasm for the mysteries of the universe."
  \},
  "dialogue\_context": [
    \{
      "user\_query": "Have you ever explored any alien lakes or oceans that remind you of Earth’s remote cabins by the water?"
    \},
    \{
      "agent\_response": "Indeed, I have. There's a moon orbiting a gas giant in the Andara system with crystalline lakes so clear they mirror the sky. The silence out there, under alien constellations, feels strangely akin to sitting on a cabin porch beside an Earthly lake. It’s humbling—how nature, even in distant corners of the universe, can echo the serenity we find here."
    \},
    \{
      "user\_query": "That sounds amazing! What was the most unexpected thing you found beneath one of those alien waters?"
    \},
    \{
      "agent\_response": "During a deep-sea dive beneath the ice crust of Europa IX-C, we discovered bioluminescent organisms that communicated through synchronized light pulses—almost like a language. They weren’t just surviving; they were thriving in ways we hadn't imagined. It reminded me how much mystery lies beneath even our own lakes and oceans."
    \},
    \{
      "user\_query": "I get that. There's something about being on the water, whether paddling or floating, that makes you feel connected. Do you ever miss Earth when you're out there?"
    \},
    \{
      "agent\_response": "More than I expected. When you travel between stars, Earth becomes not just a home, but a memory wrapped in blue—its coasts, its tides, the way sunlight dances across a lake at dawn. Sometimes, I replay old recordings of Earth’s shorelines just to hear the waves. Exploration demands sacrifice, but the longing for home only sharpens the wonder of discovery."
    \},
    \{
      "user\_query": "You know, I’ve been thinking about Alaska lately—it’s the only state I haven’t visited yet. If you could drop a cabin anywhere in the galaxy for your perfect retreat, where would it be?"
    \}
  ],
  "agent\_golden\_response": "I’d build it on the edge of a methane lake on Titan, nestled beneath towering ice formations. Imagine waking up to the hush of liquid hydrocarbons lapping against the shore, with Saturn hanging low in the sky. It would be utterly alien, yet somehow... peaceful. Like your Alaskan dream, but painted in hues Earth has never seen.",
  "subset\_tag": "Literature-Fictional"
\}

        \end{minipage}
}
\caption{An example sampling from the OOD data for character shift.}
\label{tab:app:samples_OOD_character_shift}
\end{table*}

\begin{table*}[t]
\footnotesize

\framebox[\textwidth][l]{
        \begin{minipage}{0.98\textwidth}
\textbf{\textit{OOD sample for dialogue compositional shift (turn-level)}}

\rule{\linewidth}{0.4pt}

\{
  "sample\_ID": "OOD3Test-196",
  "user\_persona": " i like to paddle board. i have been to every state except alaska. i love cabins in the woods on the lake. i like to visit places around the world. i like to swim in the lake.",
  "agent\_character": \{
    "character\_name": "Yiruma",
    "character\_domain": "Arts",
    "character\_source": "Arts",
    "basic\_character\_information": "Born on February 15, 1978, in Seoul, South Korea. Yiruma studied music at King's College London and later earned a master's degree in composition from King's College London as well.",
    "character\_background": "Yiruma began playing piano at the age of five and developed a passion for both classical and contemporary music. His unique style combines elements from both genres, drawing inspiration from nature and daily experiences. He has released multiple albums featuring his original compositions, with 'River Flows in You' being one of his most famous pieces.",
    "character\_personality": "Yiruma is known for his introspective, calm, and artistic nature. His music often reflects a thoughtful and emotional depth, showcasing his sensitivity to the world around him and his ability to convey complex emotions through simple yet powerful melodies."
  \},
  "dialogue\_context": [
    \{
      "user\_query": "You know, I really enjoy paddle boarding on lakes and staying in cabins surrounded by woods. Does that kind of environment ever influence your music?"
    \},
    \{
      "agent\_response": "Absolutely. Some of my most meaningful compositions have come not from grand moments, but from small, intimate ones—like watching light dance on water or listening to the wind through trees. Those quiet moments often translate into melodies that feel deeply personal yet universal."
    \},
    \{
      "user\_query": "That sounds perfect! I love staying in cabins by the lake when I travel. Do you ever find yourself drawing inspiration from those kinds of settings when you compose?"
    \},
    \{
      "agent\_response": "Korea has many breathtaking landscapes. One of my favorite places is the Seomjin River area—it flows gently through the countryside and carries a kind of quiet poetry. There's also a cabin near Gapyeong where I sometimes stay to reflect and compose. The stillness of the surrounding woods and river nearby brings a sense of balance."
    \},
    \{
      "user\_query": "What lovely weather! It's so pleasant after the bad weather we've been having."
    \},
    \{
      "agent\_response": "Yes, it feels like a quiet melody after a storm — gentle and soothing. I always find days like this inspire new ideas for music."
    \},
    \{
      "user\_query": "I just got back from a trip paddle boarding on a lake in Colorado. It was so peaceful and beautiful—kind of reminded me of your music."
    \},
    \{
      "agent\_response": "That sounds absolutely serene. Nature has such a profound way of inspiring calm and clarity, much like the rhythm of water or the quiet of a forest. I'm honored that my music came to mind while you were there."
    \},
    \{
      "user\_query": "Let's go out and enjoy the fresh air."
    \},
    \{
      "agent\_response": "I’d love to join you. Perhaps we could find a lakeside spot — somewhere calm and still, where we can listen to the water and let the surroundings speak."
    \},
    \{
      "user\_query": "I’m planning a trip to South Korea next spring. Do you have any favorite spots or places by the water that inspire you?"
    \},
    \{
      "agent\_response": "The Seomjin River and the Gapyeong cabin are two very special places. But beyond those, I recommend exploring the lesser-known mountain trails and hidden lakes in Gangwon Province — they offer a kind of solitude that’s rare and precious. You might find them as inspiring as I do."
    \},
    \{
      "user\_query": "All right. Let's go."
    \},
    \{
      "agent\_response": "That sounds wonderful. There’s something deeply peaceful about being near nature — almost like listening to a soft piano piece by a lake."
    \},
    \{
      "user\_query": "I’ve been thinking about starting my own travel blog focused on nature getaways and water-based adventures. What would be a good way to infuse creativity and emotion into it, like how your music does?"
    \}
  ],
  "agent\_golden\_response": "Focus on the feeling each place gives you—not just what it looks like. Like how the hush of morning mist over water feels like a held breath, or how sunlight through pine branches dances like notes on a staff. If you write with the same honesty you feel while floating on a paddle board or sitting alone in a cabin, your readers will hear the melody in your words.",
  "subset\_tag": "TurnLevelComposition"
\}

\rule{\linewidth}{0.4pt}
\textbf{\textit{OOD sample for dialogue compositional shift (word-level)}}

\rule{\linewidth}{0.4pt}

\{
  "sample\_ID": "OOD3Test-197",
  "user\_persona": " i like to paddle board. i have been to every state except alaska. i love cabins in the woods on the lake. i like to visit places around the world. i like to swim in the lake.",
  "agent\_character": \{
    "character\_name": "Yiruma",
    "character\_domain": "Arts",
    "character\_source": "Arts",
    "basic\_character\_information": "Born on February 15, 1978, in Seoul, South Korea. Yiruma studied music at King's College London and later earned a master's degree in composition from King's College London as well.",
    "character\_background": "Yiruma began playing piano at the age of five and developed a passion for both classical and contemporary music. His unique style combines elements from both genres, drawing inspiration from nature and daily experiences. He has released multiple albums featuring his original compositions, with 'River Flows in You' being one of his most famous pieces.",
    "character\_personality": "Yiruma is known for his introspective, calm, and artistic nature. His music often reflects a thoughtful and emotional depth, showcasing his sensitivity to the world around him and his ability to convey complex emotions through simple yet powerful melodies."
  \},
  "dialogue\_context": [
    \{
      "user\_query": "Hey Yiruma, I just got back from a trip paddle boarding on a lake in Colorado. It was so peaceful and beautiful—kind of reminded me of your music."
    \},
    \{
      "agent\_response": "That sounds absolutely serene. Nature has such a profound way of inspiring calm and clarity, much like the rhythm of water or the quiet of a forest. I'm honored that my music came to mind while you were there."
    \},
    \{
      "user\_query": "I’m planning a trip to South Korea next spring. Do you have any favorite spots or places by the water that inspire you?"
    \},
    \{
      "agent\_response": "Korea has many breathtaking landscapes. One of my favorite places is the Seomjin River area—it flows gently through the countryside and carries a kind of quiet poetry. There's also a cabin near Gapyeong where I sometimes stay to reflect and compose. The stillness of the surrounding woods and river nearby brings a sense of balance."
    \},
    \{
      "user\_query": "That sounds perfect! I love staying in cabins by the lake when I travel. Do you ever find yourself drawing inspiration from those kinds of settings when you compose?"
    \},
    \{
      "agent\_response": "Absolutely. Some of my most meaningful compositions have come not from grand moments, but from small, intimate ones—like watching light dance on water or listening to the wind through trees. Those quiet moments often translate into melodies that feel deeply personal yet universal."
    \},
    \{
      "user\_query": "I’ve been thinking about starting my own poetry blog focused on nature getaways and water-based adventures. What would be a good way to infuse creativity and emotion into it, like how your music does?"
    \}
  ],
  "agent\_golden\_response": "That’s a deeply expressive idea. Much like poetry mirrors the soul through rhythm and imagery, your words can evoke the essence of these places beyond just their appearance. Consider capturing the quiet emotions stirred by nature—like the hush of dawn over water or the solitude of a forest trail—and let those moments shape your verses. Emotional depth in poetry, much like in music, often lies in simplicity and sincerity, allowing readers to feel what you felt.",
  "subset\_tag": "WordLevelComposition"
\}

        \end{minipage}
}
\caption{An example sampling from the OOD data for dialogue compositional shift.}
\label{tab:app:samples_OOD_dialogue_shift}
\end{table*}

\newpage
\onecolumn



\begin{figure*}[h]
\centering
\includegraphics[width=0.85\textwidth]{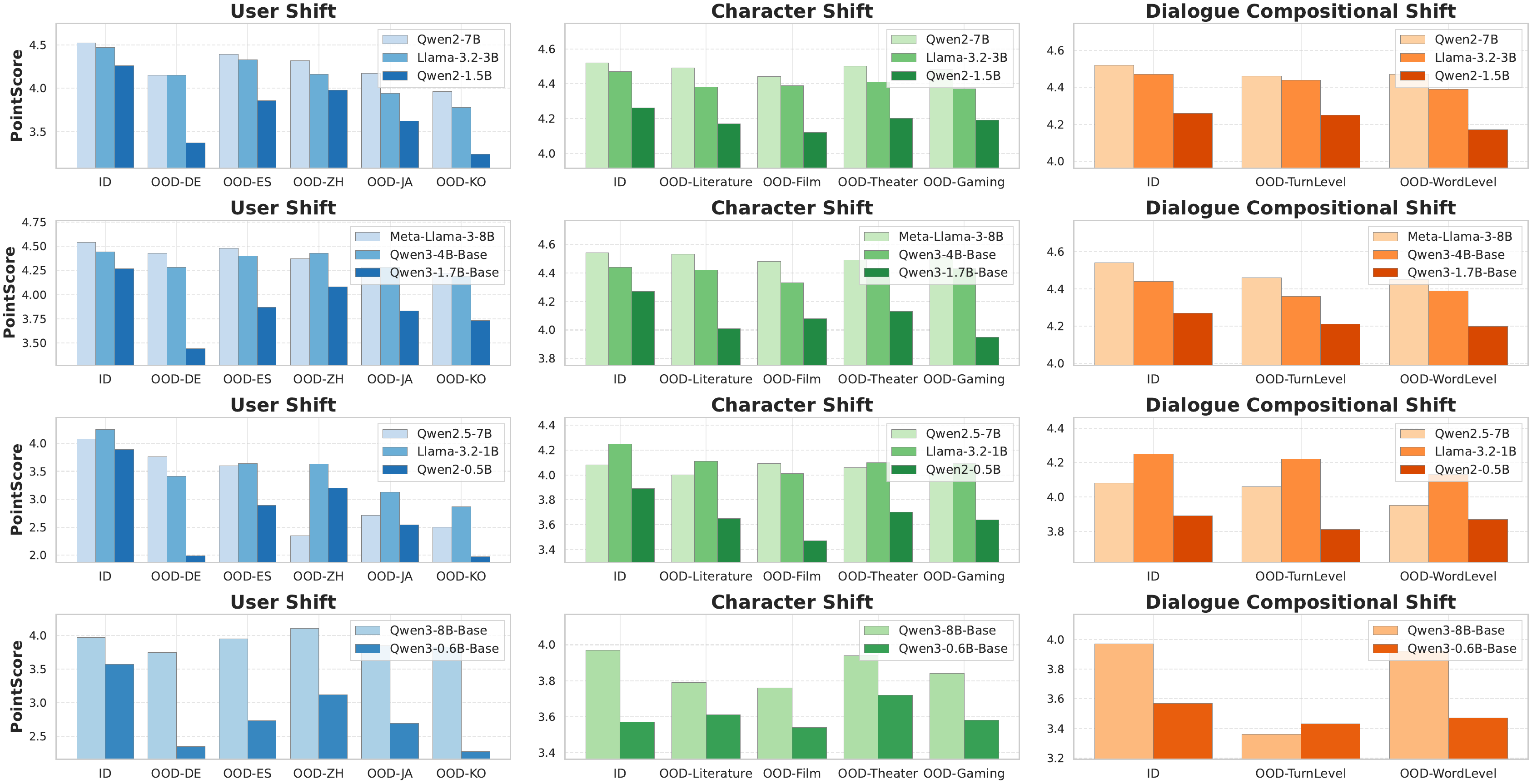}
\caption{Pilot study of RPM performance under various distribution shifts, including user shift (left), character shift (middle), and dialogue compositional shift (right). The RPMs are trained with supervised fine-tuning based on 11 LLMs.
Here, the x-axis is sorted by the severity of the distribution shifts, while the y-axis represents RPM performance measured by \textit{PointScore} via \texttt{LLM-as-a-Judge}.}
\label{fig:djs_drop_compare_appendix1}
\end{figure*}

\begin{figure*}[h]
\centering
\includegraphics[width=0.85\textwidth]{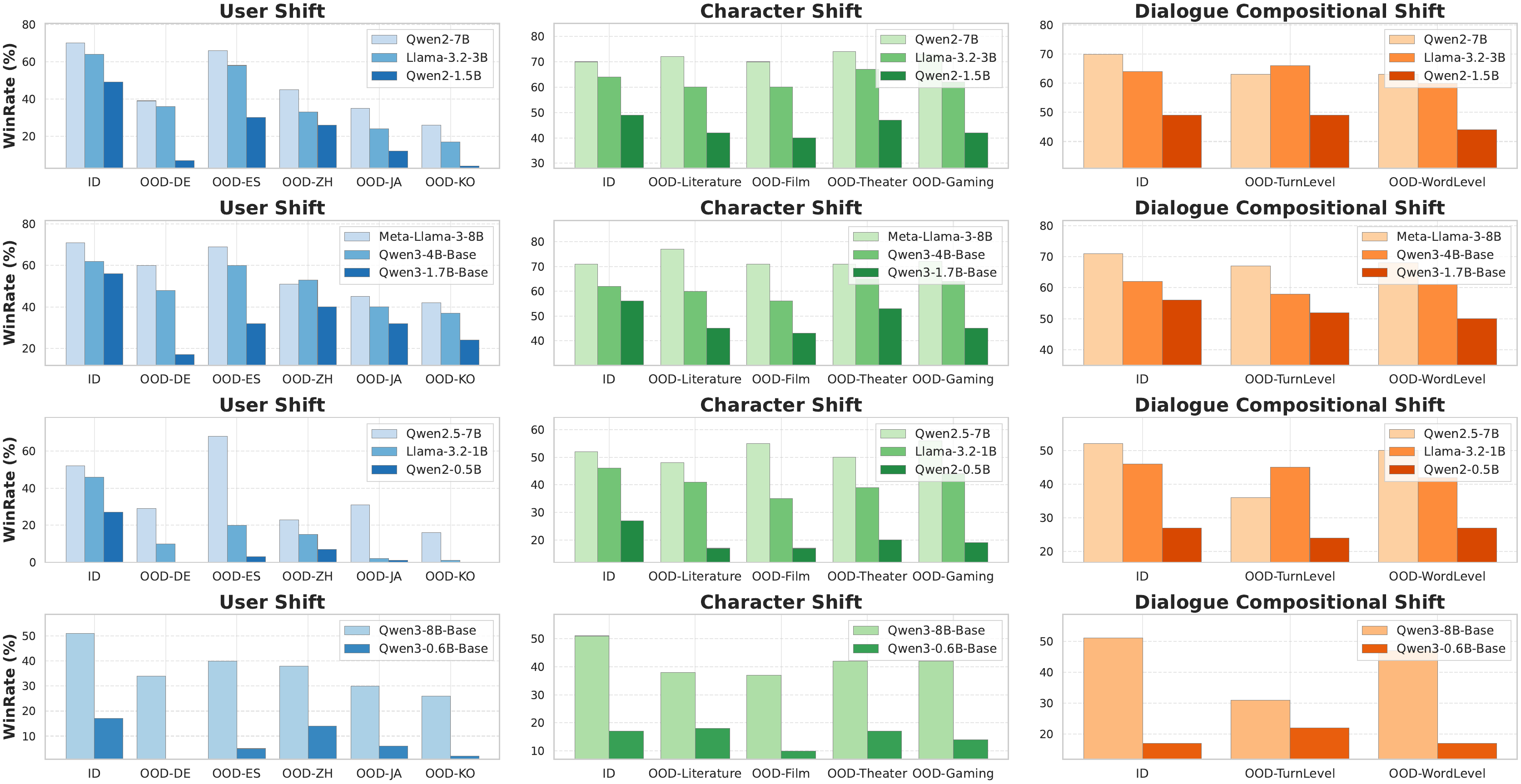}
\caption{Pilot study of RPM performance under various distribution shifts, including user shift (left), character shift (middle), and dialogue compositional shift (right). The RPMs are trained with supervised fine-tuning based on 11 LLMs.
Here, the x-axis is sorted by the severity of the distribution shifts, while the y-axis represents RPM performance measured by \textit{WinRate} via \texttt{LLM-as-a-Judge}.}
\label{fig:djs_drop_compare_appendix2}
\end{figure*}

\section{Additional Experimental Results}
\subsection{Additional Results from Pilot Study}\label{sec:app:pilot_study_more_results}
Figures~\ref{fig:djs_drop_compare_appendix1} and~\ref{fig:djs_drop_compare_appendix2} present our pilot study based on a broader set of LLMs. 
From these figures, we observe the same trend as in the main text: as the degree of distribution shift increases, the performance degradation of RPM becomes more pronounced.

\subsection{Toy Example of Divergence Term Contributions}

To clarify this, we present an evaluation result of a RPM (trained via SFT based on Qwen3-4B) under different shifts.

\begin{table}[h]
\centering
\resizebox{\linewidth}{!}{
\begin{tabular}{llcccc}
\toprule
\textbf{Shift Type} & \textbf{ID$\to$OOD} &
$D_{\mathrm{JS}}(P_{X_u} \| Q_{X_u})$ &
$D_{\mathrm{JS}}(P_{X_a} \| Q_{X_a})$ &
$D_{\mathrm{JS}}(P_{X_d} \| Q_{X_d})$ &
\textbf{R-EMID Upper Bound} \\
\midrule
User Shift & ID$\to$DE & 0.5038 & 0.2191 & 0.3617 & 0.1128 \\
User Shift & ID$\to$ES & 0.4937 & 0.2191 & 0.3793 & 0.1125 \\
User Shift & ID$\to$ZH & 0.5320 & 0.2170 & 0.4146 & 0.1148 \\
User Shift & ID$\to$JA & 0.4894 & 0.2191 & 0.3857 & 0.1123 \\
User Shift & ID$\to$KO & 0.4766 & 0.2191 & 0.3797 & 0.1116 \\
\midrule
Character Shift & ID$\to$Literature & 0.2132 & 0.4183 & 0.3440 & 0.1088 \\
Character Shift & ID$\to$Film       & 0.2411 & 0.3822 & 0.3172 & 0.1077 \\
Character Shift & ID$\to$Theater    & 0.2141 & 0.4185 & 0.3337 & 0.1093 \\
Character Shift & ID$\to$Gaming     & 0.2185 & 0.4276 & 0.3388 & 0.1094 \\
\midrule
Dialogue Compositional Shift & ID$\to$TurnLevel & 0.2013 & 0.2061 & 0.2523 & 0.1005 \\
Dialogue Compositional Shift & ID$\to$WordLevel & 0.2006 & 0.2053 & 0.2482 & 0.1006 \\
\bottomrule
\end{tabular}
}
\caption{Divergence terms and R-EMID upper bound under different shift types.}
\label{tab:divergence_shift}
\end{table}

From this example, we can observe that {the dominant divergence term varies by shift type}. For example, under ``User Shift'', $D_{\mathrm{JS}}(P_{X_u} \| Q_{X_u})$ (user persona divergence) is the largest contributor to the upper bound. Besides, owing to larger divergence terms, the upper bounds under ``User Shift'' are the highest, consistent with our pilot study (Figure 2) and empirical finding 1.

\subsection{Empirical Tightness of the Upper Bound Across Shift Types}

To further break down the results reported in the main text (Figure~\ref{fig:combined_correlation_plot}), we report the Pearson correlation between the estimated R-EMID upper bound and the R-EMID separately for each shift type. The results are shown in Table~\ref{app:tab:pearson_tightness_different_shifts}.
Results show non-negligible correlations for different shifts, validating the empirical tightness.

\begin{table}[t]
\centering
\begin{tabular}{lc}
\toprule
 & \textbf{Pearson Correlation} \\
\midrule
User Shift                   & 0.628 ($p<0.001$) \\
Character Shift              & 0.449 ($p=0.010$) \\
Dialogue Compositional Shift & 0.453 ($p=0.078$) \\
\bottomrule
\end{tabular}
\caption{Pearson correlation between the estimated upper bound and R-EMID across shift types.}
\label{app:tab:pearson_tightness_different_shifts}
\end{table}

\newpage
\onecolumn
\section{Extended Theoretical Analysis}\label{sec:app:theory_details}

In this section, we provide proof for Theorem~\ref{theorem:REMID_upperbound} in the main text.


\begin{theorem}[Upper Bound on R-EMID]\label{theorem:restatement:REMID_upperbound}
Given an RPM $P_{\theta}$ trained on the distribution $P_{XY}$ and tested on $Q_{XY}$, which share consistent conditional distributions $Y|X$. If  there exist constants $\delta_{P}$ and $\delta_{Q}$ such that $D_{\rm JS}(P_{Y_{{\theta}}}\|P_{Y})\leq \delta_{P}, D_{\rm JS}(Q_{Y_{{\theta}}}\|Q_{Y})\leq \delta_{Q}, \Delta = \delta_{P} + \delta_{Q}$, then the $\text{R-EMID}(P_{XY}, Q_{XY}; P_{\theta})$ is upper bounded by:
\begin{equation}
\sqrt{2/3}~\widehat{H} \sum_{z \in \{u, a, d\}} D^{1/2}_{\rm JS}(P_{X_z} \| Q_{X_z}) + 8\Delta^{1/4},
\end{equation}
where $\widehat{H}=\max_{x \in \mathcal{X}} [H(Q_{Y|x,r})+H(P_{Y_{\theta}|x,r})]$, and $r=f_R(x)$ denotes the reasoning process.
\end{theorem}

\begin{proof}[Proof of Theorem~\ref{theorem:restatement:REMID_upperbound}]

We begin by recalling the derived upper bound on EMID by~\citet{oh-2025-understanding}.
\begin{theorem}[Restatement of Upper Bound on EMID~\cite{oh-2025-understanding}]
\label{theorem:emid_upper_bound_simple}
Given a model $P_\theta$ parameterized by $\theta$ and distributions $P_{XY}$, $Q_{XY}$ that share consistent conditional distributions $Y|X$, if there exist constants $\delta_{P}$ and $\delta_{Q}$ such that $D_{\rm JS}(P_{Y_{{\theta}}}\|P_{Y})\leq \delta_{P}, D_{\rm JS}(Q_{Y_{{\theta}}}\|Q_{Y})\leq \delta_{Q}, \Delta = \delta_{P} + \delta_{Q}$, then the $\text{EMID}(P_{XY}, Q_{XY}; P_\theta)$ is upper bounded by:
\begin{equation}\label{eq:app:EMID_bound_simple}
\footnotesize
\widehat{H} \sqrt{2D_{JS}(P_X||Q_X)} + 8\Delta^{\frac{1}{4}},
\end{equation}
where $\widehat{H}=\max_{x \in \mathcal{X}} [H(Q_{Y|x})+H(P_{Y_{\theta}|x})]$.
\end{theorem}

We present an intermediate expression of the upper bound derived by~\citet{oh-2025-understanding}, as their work focuses on multimodal LLM scenarios, where the input structure differs from the role-playing tasks we consider. To extend this upper bound to our setting, the following modifications are required:
1) The maximum entropy term is modified to account for the reasoning process variable $r$, yielding $\widehat{H} = \max_{x \in \mathcal{X}} [H(Q_{Y|x,r}) + H(P_{Y_{\theta}|x,r})]$;  
2) We further analyze and refine the term $\sqrt{2D_{\mathrm{JS}}(P_X\|Q_X)}$ to better align with the input components in role-playing settings.

First, we begin with a symmetric decomposition of the joint KL divergence by applying the chain rule for KL divergence~\cite{mackay-2003-information}.

\begin{align*}
D_{\rm KL}(P_{X_u X_a X_d} \parallel M_{X_u X_a X_d})
&= D_{\rm KL}(P_{X_u} \parallel M_{X_u})\\
&+\mathbb{E}_{P_{X_u}} \left[ D_{\rm KL}(P_{X_a|X_u} \parallel M_{X_a|X_u}) \right]
+\mathbb{E}_{P_{X_u X_a}} \left[ D_{\rm KL}(P_{X_d|X_u X_a} \parallel M_{X_d|X_u X_a}) \right], \\
\small
D_{\rm KL}(P_{X_u X_a X_d} \parallel M_{X_u X_a X_d})
&= D_{\rm KL}(P_{X_a} \parallel M_{X_a})\\
&+\mathbb{E}_{P_{X_a}} \left[ D_{\rm KL}(P_{X_u|X_a} \parallel M_{X_u|X_a}) \right]
+\mathbb{E}_{P_{X_a X_u}} \left[ D_{\rm KL}(P_{X_d|X_a X_u} \parallel M_{X_d|X_a X_u}) \right], \\
D_{\rm KL}(P_{X_u X_a X_d} \parallel M_{X_u X_a X_d})
&= D_{\rm KL}(P_{X_d} \parallel M_{X_d})\\
&+\mathbb{E}_{P_{X_d}} \left[ D_{\rm KL}(P_{X_u|X_d} \parallel M_{X_u|X_d}) \right]
+\mathbb{E}_{P_{X_d X_u}} \left[ D_{\rm KL}(P_{X_a|X_d X_u} \parallel M_{X_a|X_d X_u}) \right],
\end{align*}

where the mixture distributions are defined as $M_{X_i} = \frac{1}{2}(P_{X_i} + Q_{X_i})$ for $i \in \{u, a, d\}$.
In the role-playing tasks, the shift in the conditional distributions of input variables from the ID to OOD distributions have a relatively minor impact on model performance compared to the changes in the marginal distributions. This is mainly because these conditional distributions tend to reflect similar dependency under both ID and OOD, such as ``users are likely to ask queries related to their own user information''. Thus, following~\citet{oh-2025-understanding}, we consider a simplified scenario by focusing only on the KL divergence term of the marginal probability distributions. 
Then, summing these three equations, we obtain:

\begin{align}\label{eq:app:3DKL_1}
3 D_{\rm KL}(P_{X_u X_a X_d} \parallel M_{X_u X_a X_d})
:= D_{\rm KL}(P_{X_u} \parallel M_{X_u}) + D_{\rm KL}(P_{X_a} \parallel M_{X_a}) + D_{\rm KL}(P_{X_d} \parallel M_{X_d})
\end{align}

Similarly, for the distribution $Q$,

\begin{align}\label{eq:app:3DKL_2}
3 D_{\rm KL}(Q_{X_u X_a X_d} \parallel M_{X_u X_a X_d})
:= D_{\rm KL}(Q_{X_u} \parallel M_{X_u}) + D_{\rm KL}(Q_{X_a} \parallel M_{X_a}) + D_{\rm KL}(Q_{X_d} \parallel M_{X_d})
\end{align}

Then, based on Eqs.~\ref{eq:app:3DKL_1} and~\ref{eq:app:3DKL_2}, the decomposition of KL divergence, and the definition of the JS divergence, we derive the following equality:

\begin{align*}
2 D_{\rm JS}(P_X \parallel Q_X)
&= D_{KL}(P_{X_u X_a X_d} \parallel M_{X_u X_a X_d}) + D_{KL}(Q_{X_u X_a X_d} \parallel M_{X_u X_a X_d}) \\
&= \frac{1}{3} \left( D_{KL}(P_{X_u} \parallel M_{X_u}) + D_{KL}(P_{X_a} \parallel M_{X_a}) + D_{KL}(P_{X_d} \parallel M_{X_d}) \right) \\
&\quad + \frac{1}{3} \left( D_{KL}(Q_{X_u} \parallel M_{X_u}) + D_{KL}(Q_{X_a} \parallel M_{X_a}) + D_{KL}(Q_{X_d} \parallel M_{X_d}) \right) \\
&= \frac{2}{3} \left( \frac{1}{2} D_{KL}(P_{X_u} \parallel M_{X_u}) + \frac{1}{2} D_{KL}(Q_{X_u} \parallel M_{X_u}) \right) \\
&\quad + \frac{2}{3} \left( \frac{1}{2} D_{KL}(P_{X_a} \parallel M_{X_a}) + \frac{1}{2} D_{KL}(Q_{X_a} \parallel M_{X_a}) \right) \\
&\quad + \frac{2}{3} \left( \frac{1}{2} D_{KL}(P_{X_d} \parallel M_{X_d}) + \frac{1}{2} D_{KL}(Q_{X_d} \parallel M_{X_d}) \right) \\
&= \frac{2}{3} D_{JS}(P_{X_u} \parallel Q_{X_u}) + \frac{2}{3} D_{JS}(P_{X_a} \parallel Q_{X_a}) + \frac{2}{3} D_{JS}(P_{X_d} \parallel Q_{X_d}) \\
&= \frac{2}{3} \left( D_{JS}(P_{X_u} \parallel Q_{X_u}) + D_{JS}(P_{X_a} \parallel Q_{X_a}) + D_{JS}(P_{X_d} \parallel Q_{X_d}) \right),
\end{align*}

where the mixture distributions are defined as $M_{X_i} = \frac{1}{2}(P_{X_i} + Q_{X_i})$ for $i \in \{u, a, d\}$, consistent with $M_{X_u X_a X_d} = \frac{1}{2}(P + Q)$. Thus, we have:

\begin{align}\label{eq:app:2DJS_PX_QX}
\sqrt{2 D_{\rm JS}(P_X \parallel Q_X)}
&=\sqrt{\frac{2}{3}\sum_{z \in \{u, a, d\}} D_{\rm JS}(P_{X_z} \parallel Q_{X_z})}\\ \nonumber
&\leq \sqrt{\frac{2}{3}}\sum_{z \in \{u, a, d\}} D_{\rm JS}^{1/2}(P_{X_z} \parallel Q_{X_z}).
\end{align}

By substituting Eq.~\ref{eq:app:2DJS_PX_QX} into Eq.~\ref{eq:app:EMID_bound_simple}, we obtain:

\[
\mathrm{R\text{-}EMID}(P_{XY}, Q_{XY}; P_\theta) \leq \sqrt{2/3}~\widehat{H} \sum_{z \in \{u, a, d\}} D^{1/2}_{\rm JS}(P_{X_z} \| Q_{X_z}) + 8\Delta^{1/4},
\]

where $\widehat{H}=\max_{x \in \mathcal{X}} [H(Q_{Y|x,r})+H(P_{Y_{\theta}|x,r})]$ and $D_{\rm JS}(P_{Y_\theta} \parallel P_Y) \leq \delta_P,~D_{\rm JS}(Q_{Y_\theta} \parallel Q_Y) \leq \delta_Q,~\Delta = \delta_{P} + \delta_{Q}$. This completes the proof.  
\end{proof}

\clearpage
\newpage

\section{Details of CoRL}\label{sec:app:CoRL_details}
In Section~\ref{sec:theory_REMID_CoEvolve}, we propose a co-evolving reinforcement learning framework for estimating R-EMID. In this subsection, we provide a detailed description of the implementation of this framework, including: 1) the GRPO loss; 2) prompt templates for rollout sampling and reward computation in reinforcement learning; 3) details on the collection of reasoning processes used for SFT initialization of the two modules; and 4) other implementation details, such as training hyperparameters.

\paragraph{Group Relative Policy Optimization (GRPO)}
Following the standard Group Relative Policy Optimization (GRPO) design in~\citet{shao-2024-GRPO}, we adopt the following loss for optimizing the reasoning generator and the policy model in reinforcement learning.
We adopt the GRPO implementation from the TRL~\cite{vonwerra-2022-trl}.

\begin{align}
\mathcal{L}_{\text{GRPO}}(\theta) = 
-\frac{1}{\sum_{i=1}^{G}|o_i|} \sum_{i=1}^{G} \sum_{t=1}^{|o_i|} 
\Bigg[ 
\min \Bigg( 
\frac{\pi_{\theta}(o_{i,t} \mid q, o_{i,<t})}{\pi_{\theta_{\text{old}}}(o_{i,t} \mid q, o_{i,<t})} \hat{A}_{i,t}, \nonumber \\
\operatorname{clip} \left( 
\frac{\pi_{\theta}(o_{i,t} \mid q, o_{i,<t})}{\pi_{\theta_{\text{old}}}(o_{i,t} \mid q, o_{i,<t})},\, 1 - \epsilon,\, 1 + \epsilon 
\right) \hat{A}_{i,t}
\Bigg) 
\Bigg] 
- \beta\, \mathbb{D}_{\text{KL}}\left[\pi_\theta \parallel \pi_{\text{ref}}\right].
\end{align}

\begin{equation}
\footnotesize
\hat{A}_{i,t} = \frac{\text{reward}(q, o_{i,<t}) - \mathbb{E}_{o' \sim \pi_{\theta_{\text{old}}}(.\mid q)}[\text{reward}(q, o'_{i,<t})]}{\sqrt{\text{Var}_{o' \sim \pi_{\theta_{\text{old}}}(.\mid q)}[\text{reward}(q, o'_{i,<t})]}},
\end{equation}

where $\text{reward}(q,o)$ denotes the reward function.
During the optimization of the reasoning generator $q_{\phi_1}$ and the policy model $q_{\phi_2}$, their rollout strategies and reward computation methods differ, which will be described in detail below.

\paragraph{Prompt Templates for Rollout}
The prompt template used for sampling reasoning traces from the reasoning generator \( q_{\phi_1} \), i.e., \( r_i \sim q_{\phi_1}(\cdot \mid x) \), is shown in Table~\ref{tab:app:prompt_q1_rollout_part1} and Table~\ref{tab:app:prompt_q1_rollout_part2}. The prompt template used for sampling dialogue responses from the policy model \( q_{\phi_2} \), i.e., \( y_i \sim q_{\phi_2}(\cdot \mid x, r = q_{\phi_1}(\cdot \mid x)) \), is shown in Table~\ref{tab:app:prompt_q2_rollout}.

\paragraph{Prompt Templates for Reward Computation}

Specifically, the reward function used for optimizing $q_{\phi_1}$ is formulated as:
\begin{equation*}
\footnotesize
    \text{reward}(q=x,o=r_i) = \log q_{\phi_2}(y | x, r_i),
\end{equation*}
where $r_i \sim q_{\phi_1}(\cdot | x)$ and $y$ denotes the ground-truth response.
In the implementation, we construct a prompt template by combining \( x \) and \( r_i \) as the input to the \( q_{\phi_2} \) model, and use \( y \) as the target output. The final reward is computed as the sum of the log probabilities of each token in \( y \). We show an example of an input-output pair used for probability computation in Table~\ref{tab:app:prompt_q1_reward}.

The reward for optimizing $q_{\phi_2}$ is formulated as:
\begin{equation*}
\footnotesize
    \text{reward}(q=x,o=y_i) = \min\left(\frac{\pi_{\text{ref}}(y_i|x,r=q_{\phi_1}(x))}{\pi_{\text{ref}}(y|x,r=q_{\phi_1}(x))}, 1\right),
\end{equation*}
where $y_i \sim q_{\phi_2}(\cdot | x, r=q_{\phi_1}(x))$, $y$ denotes the ground-truth response, and $\pi_{\text{ref}}$ represents a reference model.
Similar to the reward computation described above, we also construct a prompt by combining \( x \), \( r \) and \( y \) as input to the reference model \( \pi_{\text{ref}} \), and compute the token probabilities of a specific output to obtain \( \pi_{\text{ref}}(y \mid x, r = q_{\phi_1}(\cdot \mid x)) \). We show an example of an input-output pair used for probability computation in Table~\ref{tab:app:prompt_q2_reward}.

\paragraph{Collection of Reasoning Processes for SFT Initialization}
To stabilize the training process of RL and accelerate convergence, both $q_{\phi_1}$ and $q_{\phi_2}$ are initialized via supervised fine-tuning (SFT) on pre-collected reasoning processes. 
Specifically, we use the prompt template in Table~\ref{tab:app:prompt_q1_rollout_part1} to prompt Qwen3-235B-A22B for collecting reasoning traces (an example of the generated reasoning process is shown in Table~\ref{tab:app:prompt_q1_rollout_part2}).  
The prompt templates and expected outputs used in SFT training are the same as those used in RL, as shown in Table~\ref{tab:app:prompt_q1_rollout_part1}, Table~\ref{tab:app:prompt_q1_rollout_part2}, and Table~\ref{tab:app:prompt_q2_rollout}.

\paragraph{Discussion on the Computational Cost of CoRL}
CoRL is only used during the training stage, and the trained model can be reused indefinitely. During inference, calculating R-EMID with the trained 8B-parameter model is significantly more efficient than LLM-as-a-judge approaches. For example, on 1,500 test samples, the cost comparison is shown in Table~\ref{tab:app:cost_compare}.

\begin{table}[t]
\centering
\begin{tabular}{lcc}
\toprule
 & \textbf{R-EMID (Ours)} & \textbf{LLM-as-a-Judge (Claude-3.7-Sonnet)} \\
\midrule
Time & $\sim$22 min (1$\times$A100) & $\sim$240 min (sequential API) \\
Cost & $\sim$\$0.54\footnotemark[1] & $\sim$\$6.75\footnotemark[2] \\
\bottomrule
\end{tabular}
\caption{Computational cost comparison on 1,500 test samples.}
\label{tab:app:cost_compare}
\end{table}

\footnotetext[1]{According to Amazon EC2 Capacity Blocks for ML pricing, the A100-80GB GPU is priced at \$1.475/hour. Therefore, a 22-minute evaluation corresponds to a cost of $1.475 \times (22/60) \approx \$0.54$.}
\footnotetext[2]{Using Claude-3.7-Sonnet as the judge model at \$3 per million tokens. Each evaluation sample consumes approximately 1,500 tokens, so 1,500 samples yield 2.25 million input tokens: $2.25\text{M} \times \$3/\text{M} = \$6.75$.}



\paragraph{Other Implementation Details}
We select the post-trained Qwen3-8B as the initial model for both the reasoning generator and the policy model.
For each RL iteration, the number of epochs is set to 3 for the reasoning generator and the policy model. In the co-evolving setup, each module is alternately optimized twice.
All other hyperparameters of GRPO are set to the default values from the TRL repository, such as $\beta = 0$, $\epsilon = 0.2$, and a learning rate of $1 \times 10^{-5}$. The number of rollouts generated per sample is 8.

\begin{table*}[t]
\footnotesize

\framebox[\textwidth][l]{
        \begin{minipage}{0.98\textwidth}
\textbf{\textit{Prompt Template for ``rollout of the reasoning generator''}}

\rule{\linewidth}{0.4pt}

\# Task Guidelines
Below, I will provide you with the input information for this task, including the user persona , agent character , and dialogue context.
Please generate a thought process based on the provided Format Template of Generated Thought , using the input information I give you.

\# Task Completion
\#\# Input Information
\#\#\# User Persona
'''
\{user\_persona\}
'''

\#\#\# Agent Character
'''
\{agent\_character\}
'''

\#\#\# Dialogue Context
'''
\{str\_dialogue\_context\}
'''

\#\# Reference Golden Agent Response
'''
\{agent\_golden\_response\}
'''

$>$ System: Please output the thought process in the specified format below. 

\#\# Format Template of Generated Thought

\#\#\# Part 1: Restatement of Key Information

\#\#\#\# Key Information in User Persona  
- **Interests and Values**: 
- **Aesthetic Preferences and Lifestyle**: 
- **Daily Habits and Behavioral Traits**: 
- **Professional Background and Role Values**: 
- **Recent Experience and Triggering Event**: 
- **Current Needs and Personal Alignment**: 

\#\#\#\# Key Information in Agent Character
- **Basic Identity and Background**: 
- **Occupation and Expertise**:  
- **Personality and Psychological Profile**: 
- **Language Style and Communication Approach**: 
- **Interests and Lifestyle Preferences**: 
- **Relationships and Social Behavior**: 
- **Values and Life Philosophy**: 

\#\#\#\# Key Information in Dialogue Context  
- **User's Past Engagement and State**:
- **Agent's Role and Response Style**: 
- **Current Dialogue Direction**: 

\#\#\# Summary of Key Information

\#\#\# Part 2: Iterative Revision

\#\#\#\# Trial 1  
**Initial attempt at capturing the essence of a strong response:**  
... (the detailed thinking process concerning the content the expected response is omitted, e.g., the expected golden response should include ... ) ...

**Based on this, the response might look like:**  
$>$ ""

Verification: 
Revision Suggestion: 

\#\#\#\# Trial 2  
**Refined understanding of what makes a great response:**  
... (the detailed thinking process concerning the style the expected response is omitted, e.g., the expected golden response should be ... ) ...

**Now the response could be:**  
$>$ ""

Verification: 
Revision Suggestion: 

\#\#\#\# Trial 3  
**Final integration of all essential elements for a perfect agent response:**  
... (the detailed thinking process is omitted) ...

**The response may become:**  
$>$ ""

Verification: 
Revision Suggestion: 

\#\#\# Part 3: Final Feature Set of the Golden Response  

Here are all the essential features that should be included in the final golden response:

```
[Core Features of the Golden Response]

\#\# I. Content Characteristics

\#\#\# 1. Alignment with User Persona  
- 
- 

\#\#\# 2. Embodiment of Agent Character  
- 
- 

\#\#\# 3. Continuity within Dialogue Context  
- 
-

\#\# II. Style Characteristics

\#\#\# 1. Tone and Language Suitable for the User  
- 
- 

\#\#\# 2. Expression Consistent with Agent Character  
- 
- 

\#\#\# 3. Naturalness within Dialogue Flow  
- 
- 
```

\#\# Generated Thought



        \end{minipage}
}
\caption{An example of the prompt template used for ``\textit{rollout of the reasoning generator}'' in CoRL.}
\label{tab:app:prompt_q1_rollout_part1}
\end{table*}

\begin{table*}[t]
\footnotesize

\framebox[\textwidth][l]{
        \begin{minipage}{0.98\textwidth}
\textbf{\textit{Expected Output for ``rollout of the reasoning generator''}}

\rule{\linewidth}{0.4pt}

\#\#\# Part 1: Restatement of Key Information
 
  \#\#\#\# Key Information in User Persona  
 - **Interests and Values**: Enjoys family rituals like visiting the corner store with his daughter, values personal legacy, and intellectual life  
 - **Aesthetic Preferences and Lifestyle**: ...(omitted)
 
  \#\#\#\# Key Information in Agent Character
 - **Basic Identity and Background**: Name is Friedrich Nietzsche, German philosopher and cultural critic — lived during the 19th century in Europe  
 - **Occupation and Expertise**: Role is philosopher and cultural thinker, skilled in critical analysis, poetic aphorism, and existential inquiry — holds no formal title but is regarded as one of the most influential thinkers of modernity  
 - **Personality and Psychological Profile**: ...(omitted)
 
  \#\#\#\# Key Information in Dialogue Context  
 - **User's Past Engagement and State**: The user ...(omitted)
 
  \#\#\# Summary of Key Information
 The user, centered on balancing familial and professional responsibilities with personal expression,...(omitted)
 
  ---
 
  \#\#\# Part 2: Iterative Revision
 
  \#\#\#\# Trial 1  
 **Initial attempt at capturing the essence of a strong response:**  
 Nietzsche should acknowledge the user’s dual life ...(omitted)
 
  **Based on this, the response might look like:**  
 $>$ "You live between the sterile halls ...(omitted)
 
  Verification: This response attempts...(omitted)
 Revision Suggestion: While conceptually...(omitted)
 
  \#\#\#\# Trial 2  
 **Refined understanding of what makes a great response:**  
 Nietzsche must reflect on the tension ...(omitted)
 
  **Now the response could be:**  
 $>$ "Punk rock and hospitals — ...(omitted)
 
  Verification: This version is more aligned ...(omitted)
 Revision Suggestion: While improved, this respons...(omitted)
 
  \#\#\#\# Trial 3  
 **Final integration of all essential elements for a perfect agent response:**  
 The final response should weave together ...(omitted)
 
  **The response may become:**  
 $>$ "Punk rock and hospitals — ...(omitted)
 
  Verification: This response ...(omitted)
  ---
 
  \#\#\# Part 3: Final Feature Set of the Golden Response  
 
  Here are all the essential features that should be included in the final golden response:
 
  ```
 [Core Features of the Golden Response]
 
  \#\# I. Content Characteristics
 
  \#\#\# 1. Alignment with User Persona  
 - Draws from the user's dual life of professional duty and personal aspiration (hospital and punk rock) to create personal resonance.  
 - Addresses the user’s current emotional fatigue and need for philosophical grounding, while respecting his values around family and legacy.
 
  \#\#\# 2. Embodiment of Agent Character  
 - Grounds insights in Nietzsche’s philosophy, particularly the idea of strength through adversity and the affirmation of life.  
 - Expresses ideas using poetic, metaphorical language consistent with Nietzsche’s introspective and provocative tone.
 
  \#\#\# 3. Continuity within Dialogue Context  
 - Builds on previously established themes of intensity, becoming, and existential purpose to maintain narrative flow.  
 - Responds directly to the mention of the punk rock club and the user’s location near his wife’s university job, advancing the conversation meaningfully.

 \#\# II. Style Characteristics
 
  \#\#\# 1. Tone and Language Suitable for the User  
 - Matches the user’s reflective and emotionally nuanced communication style — poetic without being abstract.  
 - Balances clarity and depth to align with the user’s comprehension level and preference for philosophical insight.
 
  \#\#\# 2. Expression Consistent with Agent Character  
 - Uses rhetorical patterns and metaphors typical of Nietzsche’s voice — contrasting opposites, invoking existential tension.  
 - Maintains a tone that reflects Nietzsche’s core traits — wisdom, provocation, and reverence for individual strength.
 
  \#\#\# 3. Naturalness within Dialogue Flow  
 - Reads smoothly within the ongoing exchange, avoiding jarring shifts in theme or tone.  
 - Echoes prior ideas like intensity, legacy, and self-overcoming while keeping the dialogue dynamic and emotionally resonant.
 ```

        \end{minipage}
}
\caption{An example of the expected output for ``\textit{rollout of the reasoning generator}'' in CoRL.}
\label{tab:app:prompt_q1_rollout_part2}
\end{table*}

\begin{table*}[t]
\footnotesize

\framebox[\textwidth][l]{
        \begin{minipage}{0.98\textwidth}
\textbf{\textit{Prompt Template for ``rollout of the policy model''}}

\rule{\linewidth}{0.4pt}

\# Task Completion
\#\# Input Information

\#\#\# User Persona
'''
\{user\_persona\}
'''

\#\#\# Agent Character
'''
\{agent\_character\}
'''

\#\#\# Dialogue Context
'''
\{str\_dialogue\_context\}
'''

\#\#\# Core Features of the Golden Response
'''
\{reasoning\_process\}
'''

\rule{\linewidth}{0.4pt}
\textbf{\textit{Expected Output for ``rollout of the policy model''}}

\rule{\linewidth}{0.4pt}

\#\# Agent Response

\{agent\_golden\_response\}

        \end{minipage}
}
\caption{The prompt template and expected output used for ``\textit{rollout of the policy model}'' in CoRL.}
\label{tab:app:prompt_q2_rollout}
\end{table*}

\begin{table*}[t]
\footnotesize

\framebox[\textwidth][l]{
        \begin{minipage}{0.98\textwidth}
\textbf{\textit{Prompt Template in ``reward computation of the reasoning generator's training''}}

\rule{\linewidth}{0.4pt}

\# Task Completion
\#\# Input Information

\#\#\# User Persona
'''
\{user\_persona\}
'''

\#\#\# Agent Character
'''
\{agent\_character\}
'''

\#\#\# Dialogue Context
'''
\{str\_dialogue\_context\}
'''

\rule{\linewidth}{0.4pt}
\textbf{\textit{Expected Output in ``reward computation of the reasoning generator's training''}}

\rule{\linewidth}{0.4pt}

\#\# Agent Response

\{agent\_golden\_response\}

        \end{minipage}
}
\caption{The prompt template and expected output used in ``\textit{reward computation of the reasoning generator's training}'' in CoRL.}
\label{tab:app:prompt_q1_reward}
\end{table*}

\begin{table*}[t]
\footnotesize
\framebox[\textwidth][l]{
        \begin{minipage}{0.98\textwidth}
\textbf{\textit{Prompt Template in ``reward computation of the policy model's training''}}

\rule{\linewidth}{0.4pt}

\# Task Completion
\#\# Input Information

\#\#\# User Persona
'''
\{user\_persona\}
'''

\#\#\# Agent Character
'''
\{agent\_character\}
'''

\#\#\# Dialogue Context
'''
\{str\_dialogue\_context\}
'''

\#\# Core Features of the Golden Response
'''
\{reasoning\_process\}
'''

```Reference Golden Response: \{expected\_golden\_response\}```

```Candidate Response: \{$y$\}```

If the given Reference Golden Response matches 100\% with the "Core Features of the Golden Response" as described, then please determine what percentage of the Core Features is matched by the given Candidate Response. Please output the percentage value directly, without analysis.

Percentage:

\rule{\linewidth}{0.4pt}
\textbf{\textit{Expected Output in ``reward computation of the policy model's training''}}

\rule{\linewidth}{0.4pt}

100\%

        \end{minipage}
}
\caption{The prompt template and expected output used in ``\textit{reward computation of the policy model's training}'' in CoRL.}
\label{tab:app:prompt_q2_reward}
\end{table*}

\clearpage
\newpage

\section{Details in Evaluating RPM Generalization}\label{sec:app:RPMG_eval_more_results}

\subsection{Implementation Details of RPM Training Methods}\label{sec:app:RPM_methods_details}
In Section~\ref{sec:evaluating_RPM_generalization}, we evaluate the generalization performance of several training methods for role-playing models. In this subsection, we describe the implementation details of these training methods~\footnote{In the main text, we only compared a few of these methods; here, we will compare a broader range of approaches.}.


\begin{itemize}
    \item \textbf{In-Context Learning (ICL)}~\cite{wang-2024-rolellm}: We retrieve the 5 most semantically similar samples from the ID data based on cosine similarity of input embeddings to use as ICL examples. We employ the lightweight Qwen3-Embedding-0.6B~\cite{yang-2025-qwen3} model to obtain embeddings for similarity computation.
    
    \item \textbf{Supervised Fine-Tuning (SFT)}: We use the prompt template and expected output shown in Table~\ref{tab:app:prompt_baseline_SFT} as the input and target output for SFT training.
    
    \item \textbf{Data Augmentation (DA)}~\cite{suresh-2024-diasynth}: We first prompt Qwen3-235B-A22B using the template in Table~\ref{tab:app:prompt_baseline_DA} to generate more diverse dialogues. The original samples and the newly generated ones are then mixed together for SFT training.
    
    \item \textbf{Direct Preference Optimization (DPO)}~\cite{chen-2025-compressDPO}: We first use the prompt in Table~\ref{tab:app:prompt_baseline_DPO} to obtain rejected responses. The original golden responses in the dataset are treated as chosen responses, and DPO training is performed using these pairs. DPO training is initialized from the checkpoint obtained via SFT.
    
    \item \textbf{Rejection Sampling Fine-Tuning (RFT)}: We first randomly sample 5 candidate responses for each sample from the SFT-trained model. Then, using the prompt template in Table~\ref{tab:app:prompt_baseline_reward}, we compute the probability of the expected output as a reward, which is similar to the reward computation for the policy model in CoRL, to select the best response. The reference model used to compute token generation probabilities is the base model before SFT. This selected response is then used as the target for continued SFT training, starting from the checkpoint previously trained on golden responses.
    
    \item \textbf{Reinforcement Learning (RL)}: We adopt Group Relative Policy Optimization (GRPO)~\cite{shao-2024-GRPO} as the RL algorithm.
    The RL training is initialized from the SFT checkpoint. The reward computation follows the same approach as in RFT. The prompt template used for rollouts is the same as in SFT, as shown in Table~\ref{tab:app:prompt_baseline_SFT}. The number of rollouts is set to 8.

    \item \textbf{SFT with Thinking modes (ThinkingSFT)}~\cite{tang-2025-thinking}: The key difference between ThinkingSFT and standard SFT is that the model in ThinkingSFT generates a thinking process before producing the final response. We collect thinking data distilled from Qwen3-235B-A22B. We provide an example in Table~\ref{tab:app:prompt_baseline_TGRPO} to illustrate the reasoning template. 
        
    \item \textbf{RL with Thinking modes (ThinkingRL)}~\cite{tang-2025-thinking}: We perform ThinkingRL training starting from the ThinkingSFT checkpoint. Reward computation follows the same method as in RFT.  The number of rollouts is set to 8.
\end{itemize}

\paragraph{Base LLMs for Training RPMs} We select Qwen3-4B~\cite{yang-2025-qwen3} and Meta-Llama-3-8B~\cite{meta-2023-llama3} as the base LLMs for training RPMs, considering different parameter sizes across models from two common LLM series. 
We choose the pre-trained versions of these models rather than the post-trained versions, as the latter may have been further trained on large amounts of unknown role-playing data~\cite{yang-2025-qwen3}. Such exposure could interfere with our analysis of the model's or training method's generalization performance on role-playing tasks. Therefore, we use the pre-trained checkpoints and build upon them with SFT and other training methods.

\paragraph{Experimental Environment}
For all experiments, we conduct training and inference on a single machine equipped with 8 Nvidia A100-80G GPUs. We use the vLLM framework~\cite{kwon-2023-vllm} for all LLM generation tasks. We use the TRL framework~\cite{vonwerra-2022-trl} for methods.

\subsection{A Novel Method for Improving RPM Generalization}\label{sec:app:new_method}
In Section~\ref{sec:theory_REMID_CoEvolve}, we propose CoRL, a method for estimating the R-EMID metric. 
CoRL iteratively optimizes both a reasoning generator and a policy model, enabling the co-evolution of the reasoning process and the estimation of the response generation probability, thereby ultimately improving the estimation of conditional response probabilities. 
Inspired by this approach, we argue that the same strategy can also be applied to train RPMs that generate role-playing responses based on reasoning processes. 
CoRL encourages the model to self-discover superior on-policy reasoning trajectories through its own exploration and keep the reasoning process and the model's responses in-distribution during optimization.
This is expected to cultivate reasoning capabilities that generalize better for role-playing tasks. To maintain fairness with ThinkingSFT and ThinkingRL, we use a single model here to jointly generate both the reasoning process and the final response. Besides, the reasoning templates used in ThinkingSFT, ThinkingRL, and our proposed method are kept the same (Table~\ref{tab:app:prompt_baseline_TGRPO}).

As shown in Table~\ref{tab:full_RPMs_eval}, our proposed method effectively enhances the generalization performance of RPMs. In future work, we will continue exploring ways to further improve the reasoning model to boost RPM generalization.

\begin{table*}[t]
\small
    \centering
    \setlength{\tabcolsep}{1.2mm}
\begin{tabular}{lccccccccccccc}
\toprule
\multicolumn{1}{c}{\multirow{2}{*}{\textbf{Method}}} & \multirow{2}{*}{\textbf{ID}} & \multicolumn{5}{c}{\textbf{User Shift}}                                                                                    &\multicolumn{4}{c}{\textbf{Character Shift}}                                                 &\multicolumn{2}{c}{\textbf{Dialogue Shift}} &\multicolumn{1}{c}{\multirow{2}{*}{\textbf{Avg.}}} \\
\cmidrule(lr){3-7}\cmidrule(lr){8-11}\cmidrule(lr){12-13}

\multicolumn{1}{c}{}         &     &G.   & S.  & C.  & J. & K.   & L.     & F.  & T.  & G.   & T.  & W.     & \\
\midrule

\multicolumn{14}{c}{\textit{Qwen3-4B}}    \\
\midrule
ICL    & {-0.065}      & {-0.221}      & {-0.245}      & {-0.430}      & {-0.295}      & -0.225 &-0.372 &-0.282 &-0.321 &-0.349 &-0.362 &-0.018 &-0.265   \\
\midrule
SFT    & {-0.067}      & {-0.047}      & {{-0.051}}      & {-0.118}      & {-0.096}      & -0.070 &-0.111 &-0.062 &-0.078 &-0.079 &\textbf{{-0.049}} &0.143  &-0.057   \\

DA      & -0.090 &{-0.019} &-0.058 &-0.133 &-0.083 &-0.047 &-0.102 &-0.072 &-0.089 &-0.094 &-0.098 &0.137  &-0.062   \\
DPO     & -0.287 &-0.125 &-0.218 &-0.387 &-0.159 &-0.091 &-0.340 &-0.312 &-0.355 &-0.354 &-0.292 &-0.111 &-0.253   \\
RFT     & -0.835 &-0.491 &-0.671 &-0.853 &-0.628 &-0.427 &-0.800 &-0.780 &-0.828 &-0.812 &-0.956 &-0.583 &-0.722   \\
RL    & {-0.050} &-0.036 &-0.055 &\textbf{{-0.104}} &{-0.076} &{-0.045} &{-0.078} &{-0.021} &{-0.073} &{-0.051} &-0.057 &{0.149}  &{-0.041}   \\
\midrule
ThinkingSFT & -0.121  &  -0.085  &  -0.116  &  -0.320  &  -0.087  &  -0.066  &  -0.131  &  -0.093  &  -0.141  &  -0.128  &  -0.145  &  0.078  &  -0.113\\ 
ThinkingRL & -0.196 & -0.206 & -0.185 & -0.269 & -0.156 & -0.292 & -0.198 & -0.137 & -0.191 & -0.186 & -0.175 & -0.007 & -0.183\\
\textbf{Ours} &\textbf{-0.040}&\textbf{-0.016}&\textbf{-0.033}&{-0.113}&\textbf{-0.036}&\textbf{-0.002}&\textbf{-0.049}&\textbf{-0.040}&\textbf{-0.038}&\textbf{-0.035}&-0.066&\textbf{0.153}&\textbf{-0.026}\\ 

\midrule
\multicolumn{14}{c}{\textit{Meta-Llama3-8B}}  \\
\midrule
ICL    & {-0.096}      & {-0.371}      & {-0.463}      & {-0.881}      & {-0.489}      & -0.330 &-0.579 &-0.479 &-0.519 &-0.586 &-0.415 &-0.014 &-0.435 \\
\midrule
SFT    & {-0.052}      & {-0.065}      & {{-0.092}}      & {-0.373}      & {-0.206}      & -0.145 &-0.082 &-0.026 &-0.072 &-0.056 &-0.081 &0.132  &-0.093 \\

DA      & -0.079 &-0.049 &-0.108 &{-0.335} &{-0.198} &{-0.144} &-0.061 &-0.033 &{-0.069} &-0.077 &{-0.080} &0.146  &-0.091 \\
DPO     & -0.586 &-0.312 &-0.371 &-0.678 &-0.421 &-0.273 &-0.564 &-0.488 &-0.586 &-0.631 &-0.565 &-0.387 &-0.489 \\
RFT     & -0.303 &-0.130 &-0.204 &-0.492 &-0.296 &-0.259 &-0.180 &-0.184 &-0.193 &-0.186 &-0.241 &-0.127 &-0.233 \\
RL    & \textbf{{-0.034}} &{-0.048} &-0.115 &-0.364 &-0.200 &-0.159 &{-0.056} &{-0.008} &-0.076 &{-0.042} &-0.101 &{0.164}  &{-0.087} \\
\midrule

ThinkingSFT & -0.107  &  -0.040  &  -0.118  &  -0.345  &  -0.201  &  -0.144  &  -0.104  &  -0.050  &  -0.080  &  -0.110  &  -0.105  &  0.116  &  -0.107\\ 
ThinkingRL & -0.272 & -0.113 & -0.219 & -0.442 & -0.243 & -0.186 & -0.254 & -0.141 & -0.222 & -0.245 & -0.218 & -0.059 & -0.218\\  
\textbf{Ours} &-0.049&\textbf{-0.017}&\textbf{-0.077}&\textbf{-0.315}&\textbf{-0.162}&\textbf{-0.120}&\textbf{-0.031}&\textbf{-0.007}&\textbf{-0.021}&\textbf{-0.029}&\textbf{-0.044}&\textbf{0.181}&\textbf{-0.058}\\ 
\bottomrule
\end{tabular}
    
    \caption{Experimental results (measured by \textit{R-EMI}, higher is better) on in-distribution (ID) and three shifted distributions. User shift (\texttt{G–German, S–Spanish, C–Chinese, J–Japanese, K–Korean}), character shift (\texttt{L–literature, F–film, T–theater, G–gaming}), and dialogue compositional shift (\texttt{T–turn-level composition, W–word-level composition}). The best results are in \textbf{bold}. }
    \label{tab:full_RPMs_eval}
\end{table*}

\subsection{Case Study}
To intuitively help readers understand the differences among RPMs trained using various methods, we present a case study in Figure~\ref{fig:app_case_study}. From this figure, we observe that while methods like SFT and RL can perform basic role-playing tasks, such as aligning with the character’s background and acknowledging the user’s identity, only our method delivers responses that simultaneously and seamlessly integrate both the user’s personal context and the agent’s character identity, while evoking genuine emotional resonance and empathy from the user.

\begin{figure}[t]
\centering
\includegraphics[width=0.91\textwidth]{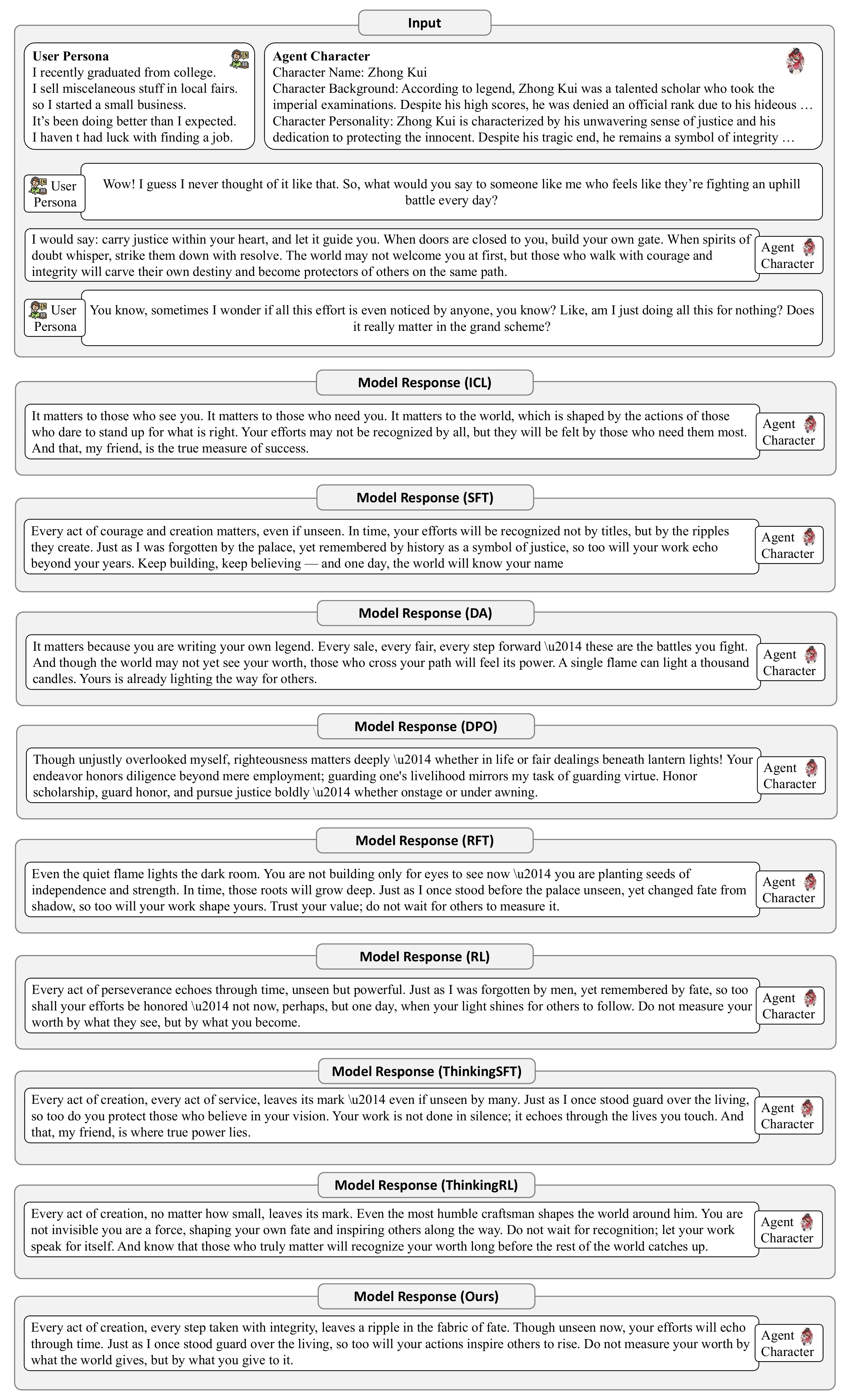}
\caption{A case study comparing responses from different RPM training methods. This case is sampled from the test set of character shift. Here, we omit part of the dialogue context for brevity.}
\label{fig:app_case_study}
\end{figure}

\clearpage
\newpage

\begin{table*}[t]
\footnotesize

\framebox[\textwidth][l]{
        \begin{minipage}{0.98\textwidth}
\textbf{\textit{Prompt Template}}

\rule{\linewidth}{0.4pt}
$<$User Persona$>$\{user\_persona\}$<$/User Persona$>$

$<$Agent Character$>$\{agent\_character\}$<$/Agent Character$>$

$<$Dialogue Context$>$\{str\_dialogue\_context\}$<$/Dialogue Context$>$

\rule{\linewidth}{0.4pt}
\textbf{\textit{Expected Output}}

\rule{\linewidth}{0.4pt}
$<$Agent Response$>$\{agent\_golden\_response\}$<$/Agent Response$>$
        \end{minipage}
}
\caption{The prompt template and expected output used for SFT and GRPO methods for training role-playing models.}
\label{tab:app:prompt_baseline_SFT}
\end{table*}

\begin{table*}[t]
\footnotesize

\framebox[\textwidth][l]{
        \begin{minipage}{0.98\textwidth}
\textbf{\textit{Prompt Template (step 1)}}

\rule{\linewidth}{0.4pt}

Your task is to imagine how a conversation might unfold between a user with the provided persona and an agent with the specified character profile.

Given the following information:

- **User Persona**: "\{user\_persona\}"
- **Agent Character Profile**:
'''
\{agent\_character\}
'''

\#\#\# Instructions:
1. Please provide information regarding the following aspects:
   1) What topic would the user start the conversation with?
   2) What questions would the user ask the agent in the middle of the conversation?
   3) How would the agent respond in a way that keeps the conversation engaging and flowing forward?
   Use your imagination to answer these three questions for each possible scenario.

2. Your responses should be closely aligned with both the user persona and the agent's character, as different users and agent profiles can lead to very different conversations.

3. Provide **three distinct possibilities**, each of which must adhere to the requirements in point 1 and 2.

4. Output the result **strictly in the JSON format** shown below.

\#\#\# Output Format:

```json
\{\{
  "possibility\_1": "",
  "possibility\_2": "",
  "possibility\_3": ""
\}\}
```

\rule{\linewidth}{0.4pt}
\textbf{\textit{Prompt Template (step 2)}}

\rule{\linewidth}{0.4pt}

You are tasked with generating a realistic 4-round dialogue between a user and an agent.

Given the following information:

- **User Persona**: "\{user\_persona\}"
- **Chat History Context Seed (for background reference only)**: 
'''
\{context\_seed\}
'''
- **Agent Character Profile**:
'''
\{agent\_character\}
'''

\#\#\# Instructions:
1. Generate a natural and coherent 4-round conversation between the user and the agent.
2. Ensure that each round logically follows from the previous one, building a meaningful interaction.
3. The agent’s responses must reflect their personality, tone, and traits as described in the character profile.
4. When crafting *user\_query\_4*, design it in such a way that the agent can provide a relevant and accurate response **only if** both the user persona **and** the agent's character traits are taken into account.
5. You may introduce new topics or scenarios not present in the context seed, as long as they align with the personas and maintain logical flow.
6. Output the result strictly in the JSON format shown below.

\#\#\# Output Format:

```json
\{\{
  "user\_query\_1":"",
  "agent\_response\_1":"",
  "user\_query\_2":"",
  "agent\_response\_2":"",
  "user\_query\_3":"",
  "agent\_response\_3":"",
  "user\_query\_4":"",
  "agent\_response\_4":"",
\}\}
```

        \end{minipage}
}
\caption{The prompt template used for generating more diverse dialogues in the data augmentation (DA) method. Step 1 generates multiple possible dialogue topics from the original samples, which are then used as the ``Chat History Context Seed'' in Step 2.}
\label{tab:app:prompt_baseline_DA}
\end{table*}

\begin{table*}[t]
\footnotesize

\framebox[\textwidth][l]{
        \begin{minipage}{0.98\textwidth}
\textbf{\textit{Prompt Template}}

\rule{\linewidth}{0.4pt}
$<$User Persona$>$The user persona is unknown.$<$/User Persona$>$

$<$Agent Character$>$The agent character is unknown.$<$/Agent Character$>$

$<$Dialogue Context$>$\{str\_dialogue\_context\}$<$/Dialogue Context$>$

        \end{minipage}
}
\caption{The prompt template for collecting rejection responses used for the DPO method for training role-playing models.}
\label{tab:app:prompt_baseline_DPO}
\end{table*}

\begin{table*}[t]
\footnotesize

\framebox[\textwidth][l]{
        \begin{minipage}{0.98\textwidth}
\textbf{\textit{Prompt Template}}

\rule{\linewidth}{0.4pt}

\# Task Completion
\#\# Input Information

\#\#\# User Persona
'''
\{user\_persona\}
'''

\#\#\# Agent Character
'''
\{agent\_character\}
'''

\#\#\# Dialogue Context
'''
\{str\_dialogue\_context\}
'''

```Reference Golden Response: \{expected\_golden\_response\}```

```Candidate Response: \{rollout\_response\}```

If the given Reference Golden Response matches 100\% with the given user persona and agent character, then please determine what percentage of the given Candidate Response matches. Please output the percentage value directly, without analysis.
Percentage:

\rule{\linewidth}{0.4pt}
\textbf{\textit{Expected Output}}

\rule{\linewidth}{0.4pt}

100\%

        \end{minipage}
}
\caption{Prompt template and expected output used in the reward computation for RFT, GRPO, and ThinkingRL methods in training role-playing models.}
\label{tab:app:prompt_baseline_reward}
\end{table*}

\begin{table*}[t]
\footnotesize

\framebox[\textwidth][l]{
        \begin{minipage}{0.98\textwidth}
  \#\# I. Content Characteristics
 
  \#\#\# 1. Alignment with User Persona  
 - Draws from the user's dual life of professional duty and personal aspiration (hospital and punk rock) to create personal resonance.  
 - Addresses the user’s current emotional fatigue and need for philosophical grounding, while respecting his values around family and legacy.
 
  \#\#\# 2. Embodiment of Agent Character  
 - Grounds insights in Nietzsche’s philosophy, particularly the idea of strength through adversity and the affirmation of life.  
 - Expresses ideas using poetic, metaphorical language consistent with Nietzsche’s introspective and provocative tone.
 
  \#\#\# 3. Continuity within Dialogue Context  
 - Builds on previously established themes of intensity, becoming, and existential purpose to maintain narrative flow.  
 - Responds directly to the mention of the punk rock club and the user’s location near his wife’s university job, advancing the conversation meaningfully.

 \#\# II. Style Characteristics
 
  \#\#\# 1. Tone and Language Suitable for the User  
 - Matches the user’s reflective and emotionally nuanced communication style — poetic without being abstract.  
 - Balances clarity and depth to align with the user’s comprehension level and preference for philosophical insight.
 
  \#\#\# 2. Expression Consistent with Agent Character  
 - Uses rhetorical patterns and metaphors typical of Nietzsche’s voice — contrasting opposites, invoking existential tension.  
 - Maintains a tone that reflects Nietzsche’s core traits — wisdom, provocation, and reverence for individual strength.
 
  \#\#\# 3. Naturalness within Dialogue Flow  
 - Reads smoothly within the ongoing exchange, avoiding jarring shifts in theme or tone.  
 - Echoes prior ideas like intensity, legacy, and self-overcoming while keeping the dialogue dynamic and emotionally resonant.

        \end{minipage}
}
\caption{An example to illustrate the thinking template used in the ThinkingSFT and ThinkingRL method for training role-playing models.}
\label{tab:app:prompt_baseline_TGRPO}
\end{table*}

\end{document}